\documentclass{article}

\usepackage[preprint]{arxiv}
\usepackage{tabularx}
\usepackage[table]{xcolor}
\usepackage{macro}
\usepackage{subcaption}
\usepackage{placeins}
\usepackage{multirow}
\usepackage{makecell}
\crefname{figure}{Figure}{Figures}

\usepackage[utf8]{inputenc} 
\usepackage[T1]{fontenc}    
\usepackage{hyperref}       
\usepackage{url}            
\usepackage{booktabs}       
\usepackage{amsfonts}       
\usepackage{nicefrac}       
\usepackage{microtype}      
\usepackage{booktabs}
\usepackage{longtable}
\usepackage{array}
\usepackage{amsmath}
\usepackage{afterpage}
\usepackage{wrapfig}

\usepackage{titlesec}
\titlespacing*{\paragraph}{0pt}{1ex}{1ex}

\title{When do prophets profit in prediction markets? \thanks{Authors are ordered alphabetically. The work by Jibang Wu is performed while visiting the University of Chicago.  }}

\author{Anri Gu \\   University of Chicago 
\\ \texttt{anrigu@uchicago.edu} \And Nicole Kagan \\   Kalshi Research \\
\texttt{nkagan@kalshi.com} \And Alec Sun \\   University of Chicago \\ \texttt{alecsun@uchicago.edu} \And Jibang Wu \\   New York University, Shanghai\\ \texttt{wujibang@nyu.edu} \And Haifeng Xu \\   University of Chicago \\   \texttt{haifengxu@uchicago.edu}}

\begin{document}

\maketitle

\begin{abstract}
Prediction markets aggregate dispersed beliefs into prices that act as probabilistic forecasts of uncertain events. Classical theory establishes how a better-than-market forecast can yield positive trading profit. However, it hinges crucially on the specific automated market maker (AMM) design, and is not applicable to popular exchanges today which are based on central limit order books. This paper fills that gap. For any prediction market and any  proper scoring rule $S$, we exhibit a ``proper'' betting strategy that depends only on the forecaster's prediction $\mathbf{p}$ and the market price $\mathbf{q}$, and earns positive expected profit \emph{whenever} $\mathbf{p}$ outperforms $\mathbf{q}$ under $S$ and the market has sufficient liquidity. Moreover, this proper betting is essentially the only strategy with such robust profitability guarantee. Our proof rests on a decomposition of expected profit that strictly generalizes the classical AMM guarantee and also explains how strategies can profit even without an accuracy edge.  Empirically, across thousands of forecasts by AI models, proper betting is the only strategy that reliably converts accuracy into profit, and we further identify systematic forecasting personas and show how the optimal proper strategy varies across them. For feasibility demonstration, we run a monthlong live pilot test on Kalshi; the encouraging preliminary results show that proper betting can survive real-world spreads, fees, discrete fills, and limited liquidity.

\end{abstract}

\textbf{Disclaimer. While this paper contains live trading results with real capital gain, these  should NOT be interpreted as trading suggestions. Instead, we conduct limited  live experiments only to test whether or not our theoretic mechanism could survive real-world spreads, fees, and limited liquidity. Professional trading  considers many more factors such as drawdown control, variance minimization, leverage,  capital allocation etc. which are out of this paper's scope.}

\section{Introduction} \label{sec:intro}

Prediction markets let people trade contracts whose prices reflect the market's estimated probability of an event, aggregating dispersed beliefs in the spirit of \citet{hayek1945use}'s ``marvel of the price system.'' Empirical studies confirm that real prediction-market prices are well calibrated and often outperform polls and expert panels \citep{berg2008results, rothschild2009forecasting, wolfers2004prediction, arrow2008promise}. The market scoring rule literature \citep{hanson2003combinatorial, hanson2007logarithmic, chen2007utility, abernethy2013efficient} formalizes the underlying incentive: an informed forecaster profits by trading against the market, and that trade pushes the price closer to the truth, so information aggregation and individual reward would go hand in hand.

This picture, however, was developed for one specific market design: an automated market maker (AMM) \citep{hanson2003combinatorial}, in which an informed forecaster ensures profit by moving the market price exactly to their own forecast, with expected profit equal to their accuracy edge. The largest prediction markets today (e.g., Kalshi, Polymarket) instead use central limit order books \citep{kalshi_rulebook, polymarket2023clob, ng2026price}, favored over AMMs for their scalability,  adaptability, and regulatory fit (see \cref{a:impl}). Prices arise from matching opposing limit orders, and a forecaster can pick any bet size but faces whatever price impact the order book happens to deliver, plus bid-ask spreads, finite liquidity, and platform fees. The link between accuracy and profit is no longer pinned down by the price-formation rule, and converting accuracy into profit becomes a separate problem.

Recent empirical work makes this gap visible. Forecasters whose predictions outperform the market on average routinely fail to convert that edge into trading profit \citep{della2026profits, jang2025losing}. On the AI leaderboards that turns AI forecasts into prediction-market trades, every agent loses money  despite several beating the market under standard proper scoring rules \citep{yang2025llm}. One may attribute the failure to the market frictions, but even in the idealized setting (with no price impact, bid-ask spread, or fees), we can observe that natural betting heuristics would produce negative profit from reasonably good forecasts (see the failure of Kelly in \cref{ex:kelly-lose}, the failure of max-margin-based betting in \cref{ex:accurate-lose} and more empirical results in \cref{sec:compare}). Stranger still, the converse also fails: simple strategies earn positive expected profit even when the underlying forecaster has no accuracy edge at all (see \cref{ex:inaccurate-win}). Translating accuracy into profit is therefore a problem in its own right, distinct from forecasting accuracy. 

\vspace{1mm}
\textbf{The primary contribution of this paper} is to establish a principled connection between predictive accuracy and robust trading profitability in general prediction markets with an arbitrary price-impact function. First, we exhibit a betting strategy that converts an accuracy edge over market into profits:
\begin{theorem*}[Informal version of \cref{thm:betting}]
If a forecaster could outperform a sufficiently liquid prediction market under any strictly proper scoring rule $S$, then there exists a specific ``proper'' betting strategy associated to $S$ that ensures positive expected profit for the forecaster. 
\end{theorem*}
The expected profit of this proper betting strategy decomposes into three explicit terms --- score gap, Bregman divergence, and liquidity loss --- which explain how strategies can profit even \emph{without} an accuracy edge (the Bregman divergence alone can suffice). Furthermore, when instantiated in  AMM markets, our proper betting strategy corresponds precisely to the canonic betting strategy in AMM that moves market price to the forecaster's belief. The above result then recovers the classic AMM guarantee as a boundary case where   the Bregman divergence exactly absorbs liquidity loss. Notably, the nature of  \cref{thm:betting} is a robust  optimization result, while not expected profit maximization which is not this paper's goal. 

Our second main result establishes that the above proper betting   is essentially the only betting strategy that robustly guarantees profitability on any sufficiently liquid market --- in the sense of ensuring positive profit \emph{whenever} there is a score edge.   

\begin{theorem*}[Informal version of \cref{thm:betting-unique}]
Any robustly profitable betting strategy is ``essentially'' (see formal Definition \ref{def:essential-diff}) the same as the above proper betting, up to  profitability-invariant modifications such as proportionally rescaling bets and constant shifting.  
 \end{theorem*}


The two main results above also yield a new characterization of proper scoring rules through the  lens of robust profitability -- that is, a scoring rule is proper \emph{if and only if} there exists an associated   betting strategy that is robustly profitable (\cref{thm:proper}). Finally, we further extend our results to three practically important settings: betting on a sequence of events based on empirical scores (instead of  the theoretical expectation), betting under nonzero bid-ask spreads, and long-horizon trading in which   the forecaster can buy and sell the same event's contracts at different time as  both the forecaster's and market's predictions evolve. 

Because testing our strategies requires vast amounts of prediction data, we experimentally evaluate our proper betting strategies on thousands of AI-generated forecasts and find a clear advantage over standard heuristic baselines. We identify systematic forecasting personas across models, with different personas favoring different proper betting strategies. Finally, to demonstrate the practical feasibility of our theory, we deployed an AI forecasting agent with real capital on the Kalshi prediction market for a month. The agent achieved very encouraging return on investment (ROI) and Sharpe ratio, which shows that our theoretically derived betting strategy can survive real-world spreads, fees, discrete fills, and limited liquidity.

\subsection{Additional related work}

\paragraph{Proper scoring rules.} Proper scoring rules originate in the statistical forecasting literature as a mechanism for eliciting truthful probabilistic
predictions \citep{brier1950verification,good1952rational,mccarthy1956measures,savage1971elicitation}. 
\citet{gneiting2007strictly} formalize strictly proper scoring rules and their relationship to Bregman divergences and convex analysis. Subsequent work extends these ideas to more general prediction settings, including decision-theoretic elicitation \citep{lambert2008eliciting},  mechanism design \citep{chen2010new}, and optimal information acquisition \citep{li2022optimization}. 

\textbf{Prediction markets.} Prediction markets themselves are a well-studied mechanism for aggregating dispersed beliefs into a consensus probability
\citep{wolfers2004prediction,arrow2008promise}. \citet{hanson2003combinatorial,hanson2007logarithmic} introduced market scoring rules, which use a proper scoring rule to automate market making and provide subsidized liquidity; this framework underpins several implementation proposals
such as the Logarithmic Market Scoring Rule (LMSR) and its derivatives \citep{chen2007utility, abernethy2013efficient}.
A parallel line of work examines price formation and information aggregation under strategic trading
\citep{chen2010gaming}.
Empirical studies of real prediction markets \citep{berg2008results,rothschild2009forecasting} have documented that prices are calibrated and often outperform
expert forecasts.

\textbf{Betting strategies.} The problem of translating a probabilistic forecast into a bet size has a long history, beginning with the Kelly criterion  \citep{kelly1956new,thorp1975portfolio,maclean2011kelly},  which maximizes long-run log-wealth under known margin. Our work inherently is also about betting strategy design, but ours differs from Kelly in a key aspect: our strategy is oblivious to the ground truth probabilities and comes with robustness guarantees, whereas Kelly criterion assumes perfect knowledge of the (usually unobservable) ground truth probabilities and  can perform poorly otherwise (see \cref{ex:kelly-lose}). 
More recent work has connected portfolio selection to online learning and proper scoring rules \citep{cover1991universal,abernethy2013efficient}, but typically assumes a single scoring-rule objective rather than asking which scoring rule best converts predictions into profit. At the same time, recent studies have highlighted a paradox in which models achieve strong predictive performance yet fail to generate positive returns \citep{della2026profits,jang2025losing}.

\section{Preliminaries}\label{sec:prelim}
\textbf{Notation.}
Define $[K] := \{1, \ldots, K\}$. Let $\Delta_{[K]} := \{\mbf{p}\in [0,1]^K : \sum_{k\in [K]} p_k = 1\}$ denote the probability simplex, $\one_y\in \{0,1\}^K$ the one-hot vector with a $1$ in $y$-th entry, $\one\in \bR^k$ the vector with every entry equal to 1, $\nab G$ the gradient of a function $G$, and $\norm{\cd}$ the Euclidean norm.

\paragraph{Prediction markets.}
Consider a probabilistic event with $K$ \emph{disjoint} outcomes and \emph{ground truth} $\mbf{p^*}\in \Delta_{[K]}$. A prediction market on this event offers, for each outcome $k\in [K]$, a \emph{contract} that pays $1$ if outcome $k$ occurs and $0$ otherwise. Let $\mbf{q} = (q_k)_{k\in [K]}\in [0,1]^K$ denote the price vector, where $q_k$ is the price of the contract on outcome $k$. A forecaster's \emph{betting strategy} is described by a position vector $\mbf{s}\in \bR^K$, where $s_k$ is the forecaster's position on outcome $k$ ($s_k>0$ buys, $s_k<0$ sells). When the outcome $y\sim \mbf{p^*}$ is realized, the forecaster receives payout $\mbf{s}\cd \one_y = s_y$. 

For ease of presentation, we assume the price vector lies on the probability simplex, $\mbf{q}\in \Delta_{[K]}$, equivalently $\sum_{k\in [K]} q_k = 1$. 
Under this condition, the price vector $\mbf{q}$ admits a natural interpretation as the market's \emph{consensus belief} \citep{arrow2008promise}. Thus a forecaster who believes that outcome $k$ realizes with probability exactly $q_k$ has zero expected return from this contract, i.e.,  $-q_k + q_k\cd 1 = 0$. 
We note that any arbitrage-free market must have $\sum_{k\in [K]} q_k = 1$, otherwise a trader can obtain a riskless profit of $\ab{1 - \sum_{k\in [K]} q_k}$ by buying or selling one share of every contract. Real markets, however, sidestep this by maintaining separate price for bid/buy and ask/sell, the gap of which is called the \emph{spread}. \cref{sec:extensions} will discuss how our main results naturally extend to the case of non-zero bid-ask spread with minor modifications.


A market has finite liquidity at the quote price $\mbf{q}$, described by a \emph{price-impact function} $\rho:\bR^K \to \Delta_{[K]}$, where $\rho(\mbf{s})$ is the post-trade spot price upon executing the \emph{position vector} $\mbf{s}$, satisfying the boundary condition $\rho(\mbf{0}) = \mbf{q}$ and the monotonicity condition $(\rho(\mbf{s}) - \rho(\mbf{s}'))\cd(\mbf{s} - \mbf{s}') \ge 0, \forall \mbf{s}, \mbf{s}' \in \bR^K$, i.e., walking the order book in any direction shall move the marginal price in that same direction. 
Executing $\mbf{s}$ incurs the integrated cost $\int_0^1 \rho(t\mbf{s})\cd \mbf{s}\, dt$, which can be decomposed into the spot-price piece $\mbf{s}\cd\mbf{q}$ plus a non-negative \emph{liquidity loss} (slippage) capturing the price impact along the trajectory:\footnote{Due to monotonicity condition of $\rho(\mbf{s})$, $\int_0^1 \rho(t\mbf{s})\cd \mbf{s}\, dt$ is the least cost for purchasing $\mbf{s}$ shares even $\rho(\mbf{s})$ may not necessarily be a conservative vector field.}
\begin{equation} \tag{Liquidity Loss}
L_{\rho}(\mbf{s}; \mbf{q}) \;:=\; \int_0^1 \rho(t\mbf{s})\cd \mbf{s}\, dt - \mbf{s}\cd \mbf{q} .
\end{equation} 
The monotonicity condition $(\rho(\mbf{s}) - \rho(\mbf{s}'))\cd(\mbf{s} - \mbf{s}') \ge 0$ implies that liquidity loss is always non-negative. Generally, the more liquid a market is, the smaller $L_{\rho}(\mbf{s}; \mbf{q})$ is (often approaching $0$ in practice for reasonably liquid markets and sized bets). 
The forecaster collects a payout $\mbf{s}\cd \one_y$ for a realized outcome $y$, so the expected profit is
\begin{equation*} \tag{Expected Profits}
\pi(\mbf{s}, \mbf{p}^*) \;:=\; \mbf{s}\cd (\mbf{p}^* - \mbf{q}) - L_{\rho}(\mbf{s}; \mbf{q}).
\end{equation*} 
\begin{remark}
 The form of $\rho$ depends on the market design: the classic automated market makers (AMMs) \citep{hanson2003combinatorial} set $\rho = \nab C$ for a convex cost function $C$, whereas the largest prediction markets today run a central limit order book for which $\rho$ is piecewise constant and read directly off the order book \citep{kalshi_rulebook,polymarket2023clob}. We discuss these two main prediction-market implementations in detail in \cref{a:impl}, but our theory is applicable to any design. We also note that $\rho$ itself evolves over time as the order book updates with new resting orders and executions; we suppress this time dependence since each of our analyses concerns a single instantaneous bet against a snapshot of the book.
\end{remark}

\paragraph{Scoring rules.}
\emph{Scoring rules} are a standard tool to measure the quality of a probabilistic forecast. 
Given a forecaster's prediction $\mbf{p}\in\Delta_{[K]}$ and a realized outcome $y\in[K]$, a \emph{scoring rule} $S(\mbf{p}, y)$ assigns a real-valued reward to the forecast $\mbf{p}$ based on  $y$, with expected score $S(\mbf{p};\mbf{p^*}) := \bE_{y\sim \mbf{p^*}}[S(\mbf{p}, y)]$ under a ground truth $\mbf{p^*}$. Scoring rules were   introduced to truthfully elicit forecasts; a scoring rule is said to be \emph{proper} if it  incentivizes a forecaster to report their true distributional belief  $\mbf{p}$.  

\begin{definition}[Proper scoring rule]
\label{def:proper-score}
    A scoring rule $S$ is \emph{proper} if a forecaster maximizes her expected score by reporting her true belief $\mbf{p}$. Formally, for every $\mbf{p}, \mbf{p}'\in \Delta_{[K]}$,
    $S(\mbf{p}; \mbf{p}) \ge S(\mbf{p}'; \mbf{p}).$ We say the rule is \emph{strictly proper} if the inequality is strict for $\mbf{p}'\ne \mbf{p}$.
\end{definition}

A proper scoring rule lets us compare any two probabilistic forecasts under some ground truth $\mbf{p^*}$. For example, applying $S$ to the market consensus $\mbf{q}$ gives the \emph{market score} $S(\mbf{q};\mbf{p^*})$, and we say the forecaster \emph{outperforms} the market if $S(\mbf{p};\mbf{p^*}) > S(\mbf{q};\mbf{p^*})$.  Importantly, the comparison of scores depends on the ground truth $\mbf{p^*}$, which can never be observed in most market applications. \cref{sec:extensions} discusses how to compare forecasts based on empirical score $S(\mbf{p};y)$ on realized outcome $y$, averaged across markets.     

Every proper scoring rule corresponds  to a convex \emph{potential function} $G$ as stated in a folklore result below.  
\begin{proposition}[Characterization of proper scoring rules \citep{mccarthy1956}] \label{prop:scoring-rule-characterization}
    A scoring rule $S$ is (strictly) proper if and only if there exists a (strictly) convex function $G: \Delta_K \to \mathbb{R}$ associated to $S$ such that $$S(\mbf{p}, y) = G(\mbf{p}) + \nab G(\mbf{p}) \cd (\one_y - \mbf{p}).$$ 
\end{proposition}

Meanwhile,   the \emph{Bregman divergence} is defined for every strictly convex differentiable function $G$ as 
$$D_G(\mbf{q}, \mbf{p}) := G(\mbf{q}) - G(\mbf{p}) - \nab G(\mbf{p})\cd (\mbf{q} - \mbf{p}),$$
where $D_G(\mbf{q}, \mbf{p}) \ge 0$ with equality if and only if $\mbf{p} = \mbf{q}$.  \cref{tab:scoring-rules} below lists  examples of proper scoring rules along their associated convex functions (the last column $s_G$ is introduced in \cref{def:proper-bet}). 

\begin{table*}[tbh]
\centering
\small
\setlength{\tabcolsep}{4pt}
\caption{Common proper scoring rules with their associated functions.}
\begin{tabular}{@{}lcccc@{}}
\toprule
 & $S(\mbf{p}, y)$ & $G(\mbf{p})$ & $D_G(\mbf{q},\mbf{p})$ & $\mbf{s}_G(\mbf{p},\mbf{q})$ \\
\midrule
\textbf{Quadratic (Brier)}
  & $-\norm{\one_{y} - \mbf{p}}^2$
  & $\norm{\mbf{p}}^2 - 1$
  & $\norm{\mbf{q} - \mbf{p}}^2$
  & $2(\mbf{p} - \mbf{q})$ \\
\textbf{Logarithmic}
  & $\log p_y$
  & $\sum_k p_k \log p_k$
  & $\sum_k q_k \log(q_k/p_k)$
  & $\log(\mbf{p}) - \log(\mbf{q})$ \\
\textbf{Spherical}
  & $p_y / \norm{\mbf{p}}$
  & $\norm{\mbf{p}}$
  & $\norm{\mbf{q}} - \mbf{p}\cd \mbf{q} / \norm{\mbf{p}}$
  & $\mbf{p}/\norm{\mbf{p}} - \mbf{q}/\norm{\mbf{q}}$ \\
\bottomrule
\end{tabular}
\label{tab:scoring-rules}

\end{table*}

\section{A theory of proper betting   for prediction markets}
\label{sec:theory}

The goal of this paper is to translate a forecaster's prediction $\mbf{p}$ that outperforms the market into a concrete betting strategy that guarantees profit in a general prediction market. There are several natural strategies, such as betting on the highest-margin outcome, scaling inversely with margin, or using Kelly criterion \citep{kelly1956new}.  The \emph{highest-margin} strategy  places a bet on the outcome $k$ for which $\ab{p_k - q_k}$ is greatest, as this outcome intuitively has the most potential for profit: the difference between the perceived value $p_k$ of the share and the actual price $q_k$ is greatest. Perhaps surprisingly, we show below that both the Kelly criterion and the highest-margin strategy do  \emph{not} guarantee positive expected profit, thus demonstrating that betting optimization is   challenging even for a good forecaster. In  \cref{sec:experiments},  we also empirically observe that all of these approaches perform poorly in practice.  

\begin{example}[Accurate forecaster + Kelly can lose]
\label{ex:kelly-lose}
    Consider an event with three outcomes and let $\mbf{p} = (0.64, 0.32, 0.04)$, $\mbf{q} = (0.80, 0.10, 0.10)$, and $\mbf{p^*} = (0.10, 0.10, 0.80)$. Under the quadratic scoring rule, since $\norm{\mbf{p^*} - \mbf{p}} < \norm{\mbf{p^*} - \mbf{q}}$, the forecaster outperforms the market. The Kelly criterion would spend $p_i$ to purchase $p_i/q_i$ shares of contract $i$ for each $i=1, 2, 3$,\footnote{This is the generalized solution for many-outcome Kelly criterion. The familiar two-outcome Kelly formula of spending $ \frac{p_1 - q_1}{1 - q_1}$ to bet on \emph{outcome} $1$ (short it if negative) is a special case.  To see this, observe that, given $\mbf{q} \in \Delta_K$, betting $ \frac{p_1 - q_1}{1 - q_1} = p_1 -   q_1 \frac{ 1 - p_1 }{1 - q_1} $ on \emph{outcome 1} and $0$ on \emph{outcome 2} is equivalent to additionally buying $\frac{ 1 - p_1 }{1 - q_1}$ shares of both \emph{outcome 1} at price $q_1$ and \emph{2} at price $1-q_1$, leading to total bet amount $p_1, p_2$ on \emph{outcome} $1,2$ respectively.} and its expected profit is $ (\mbf{p} \oslash \mbf{q}) \cd (\mbf{p^*} - \mbf{q}) < 0$. The intrinsic reason of this loss is  that the Kelly criterion is derived by assuming its forecast is the ground truth, which in practice can be very off from the forecast.    
\end{example}

\begin{example}[Accurate forecaster + highest-margin trade can lose]
\label{ex:accurate-lose}
    Consider an event with three outcomes and let $\mbf{p} = (0.61, 0.39, 0.00)$, $\mbf{q} = (0.10, 0.89, 0.01)$, and $\mbf{p^*} = (0.05, 0.10, 0.85)$. Under the quadratic scoring rule, since $\norm{\mbf{p^*} - \mbf{p}} < \norm{\mbf{p^*} - \mbf{q}}$, the forecaster outperforms the market. If the forecaster only places a bet on the highest-margin first outcome, their expected profit is $(1, 0, 0) \cd (\mbf{p^*} - \mbf{q}) < 0$.
\end{example}

On the other hand, a forecaster can profit even if their prediction underperforms the market:
\begin{example}[Inaccurate forecaster can profit]
\label{ex:inaccurate-win}
Consider an event with two outcomes and let $\mbf{p} = (0.9, 0.1)$, $\mbf{q} = (0.5, 0.5)$, and $\mbf{p^*} = (0.6, 0.4)$. Under the quadratic scoring rule, since $\norm{\mbf{p^*} - \mbf{p}} > \norm{\mbf{p^*} - \mbf{q}}$, the forecaster underperforms the market. If the forecaster places a bet on the first outcome, which is a highest-margin outcome, their expected profit is $(1, 0) \cd (\mbf{p^*} - \mbf{q}) > 0$.
\end{example}

This paper resolves this seemingly broken link between statistical accuracy and profitability in the examples above.
We first define, for every proper scoring rule $S$, what we call a \emph{proper} betting strategy corresponding to $S$. 

\begin{definition}[Proper betting strategy]
\label{def:proper-bet}
    Under a proper scoring rule $S$ with potential function $G$, the \emph{proper} betting strategy for a forecast $\mbf{p}$ and a market price $\mbf{q}$ is the position vector $$\mbf{s}_G(\mbf{p},\mbf{q}) := \nab G(\mbf{p}) - \nab G (\mbf{q}).$$
\end{definition}
Note that if some $k$'th entry of $\mbf{s}_G(\mbf{p},\mbf{q})$ is negative, the corresponding execution is to buy a \emph{NOT $k$} contract at the price of $1-q_k$ (assuming no bid-ask spread for now).

\subsection{Proper betting and its robust profitability} 
 \textbf{On robust profitability.} We now establish a formal link between a good forecast that outperforms the market and its betting profitability. A fundamental challenge in designing good betting strategies  is that we almost never observe the ground truth probability $\bm{p}^*$. Even when a forecast $\bm{p}$  is promised to have better score than the market $\bm{q}$ --- in the sense that $S(\mbf{p};\mbf{p^*}) \geq S(\mbf{q};\mbf{p^*})$ for some  proper scoring rule $S$ ---  this can hold true for infinitely many possible $\mbf{p^*}$'s. Thus, a desirable property of a good betting strategy is \emph{robust profitability}: it is guaranteed to have positive profit  for any $\mbf{p}, \bm{p}^*, \mbf{q}$ satisfying $S(\mbf{p};\mbf{p^*}) \geq S(\mbf{q};\mbf{p^*})$.  Conversely, we say a strategy $\bm{s}(\mbf{p},\mbf{q})$  is not robustly profitable under $S$, if there exists some $\mbf{p} \not = \mbf{q}, \mbf{p}^*$ such that   $S(\mbf{p};\mbf{p^*}) \geq  S(\mbf{q};\mbf{p^*})$ but $\bm{s}(\mbf{p},\mbf{q})$ has non-positive expected profit under $\mbf{p}^*$. Note that ``profitability'' here is a binary notion: it concerns only whether positive profit can be guaranteed, not the magnitude of that profit. 

Our first main result shows that  robustly profitable betting strategy exists whenever the market's liquidity loss is relatively small; in fact, the proper betting $\bm{s}_G(\mbf{p},\mbf{q})$ as defined in \cref{def:proper-bet} is such a strategy.   

\begin{theorem}[Robust profitability of proper betting]
\label{thm:betting}
For any ground truth $\mbf{p^*}$, forecast $\mbf{p}$, market price $\mbf{q}$, and strictly proper scoring rule $S$ with potential function $G$, the expected profit $\pi(\mbf{s}^*, \mbf{p}^*)$ of the proper betting strategy $\mbf{s}^* =\mbf{s}_G(\mbf{p},\mbf{q})$ admits the following decomposition: 
\begin{equation}\label{eq:profit-decomp-main}
\pi(\mbf{s}^*, \mbf{p}^*) \;=\; \ub{S(\mbf{p};\mbf{p^*}) - S(\mbf{q};\mbf{p^*})}_{\text{score gap}} \;+\; \ub{D_G(\mbf{q}, \mbf{p})}_{\text{Bregman divergence}} \;-\; \ub{L_{\rho}(\mbf{s}^*;\mbf{q})}_{\text{liquidity loss}}.   
\end{equation}

Therefore, if  $D_G(\mbf{q}, \mbf{p}) \geq L_\rho (\mbf{s}^*;\mbf{q}) $ --- i.e.,  the divergence offsets liquidity loss --- then proper betting is \emph{robustly profitable} in the sense that it ensures positive  profit whenever $S(\mbf{p};\mbf{p^*}) > S(\mbf{q};\mbf{p^*})$. 
\end{theorem}
\begin{proof} The proof hinges on an interesting  profit decomposition lemma  below that may be of independent interest. 
\begin{lemma}[Profit decomposition lemma] \label{lem:decomposition}
    For any realized outcome $y$, forecast $\mbf{p}$, market price $\mbf{q}$, and scoring rule $S(\mbf{p}, y)$ with potential function $G$ (not necessarily proper), we have
    \begin{equation}\label{eq:decomp-lemma}
        \ub{\mbf{s}_G(\mbf{p},\mbf{q}) \cd (\one_y - \mbf{q})}_{\text{Idealized profit from outcome $y$}} = \ub{S(\mbf{p},y) - S(\mbf{q},y)}_{\text{Realized score gap}} + \ub{D_G(\mbf{q}, \mbf{p})}_{\text{Bregman divergence}}. 
    \end{equation}
\end{lemma}
\begin{proof}[Proof of \cref{lem:decomposition}]
We express 
    \begin{align*}
        S(\mbf{p},y) - S(\mbf{q},y)
        &= \bp{G(\mbf{p}) - G(\mbf{q})} + \nab G(\mbf{p}) \cd (\one_y - \mbf{p})
         - \nab G(\mbf{q}) \cd (\one_y - \mbf{q})
        \\&= \bp{-\nab G(\mbf{p}) \cd (\mbf{q} - \mbf{p}) - D_G(\mbf{q}, \mbf{p})} + \nab G(\mbf{p}) \cd (\one_y - \mbf{p}) - \nab G(\mbf{q}) \cd (\one_y - \mbf{q})
        \\&= (\nab G(\mbf{p}) - \nab G(\mbf{q})) \cd (\one_y - \mbf{q}) - D_G(\mbf{q}, \mbf{p})
        \\&=  \mbf{s}_G(\mbf{p},\mbf{q})  \cd (\one_y - \mbf{q}) - D_G(\mbf{q}, \mbf{p})
    \end{align*}
    where $$D_G(\mbf{q}, \mbf{p}) = G(\mbf{q}) - G(\mbf{p}) - \nab G(\mbf{p}) \cd (\mbf{q} - \mbf{p})$$ is the \emph{Bregman divergence} between $\mbf{p}$ and $\mbf{q}$.
\end{proof}

Taking expectation  of  Equation \eqref{eq:decomp-lemma} over $y \sim \mbf{p^*}$    yields the (frictionless) expected profit of the proper bet,
    $$\mbf{s}_G(\mbf{p},\mbf{q}) \cd (\mbf{p^*} - \mbf{q}) \;=\; \ub{S(\mbf{p};\mbf{p^*}) - S(\mbf{q};\mbf{p^*})}_{\text{score gap}} \;+\; \ub{D_G(\mbf{q}, \mbf{p})}_{\text{Bregman divergence}}.$$
    
Notably, the term $\mbf{s}_G(\mbf{p},\mbf{q}) \cd (\mbf{p^*} - \mbf{q})$ is not the expected profit $\pi(\mbf{s}^*, \mbf{p}^*) $  yet since it counts the betting cost using the same market price $\mbf{q}$ for purchasing $\mbf{s}_G(\mbf{p},\mbf{q})$ contract shares, ignoring the slippage. This gap is captured precisely by the third term in \cref{eq:profit-decomp-main}  $L_{\rho}(\mbf{s}; \mbf{q})  =  \int_0^1 \rho(t\mbf{s})\cd \mbf{s}\, dt - \mbf{s}\cd \mbf{q}$. These together prove the theorem.     
\end{proof}

Like the profit decomposition of  \cref{lem:decomposition},   \cref{thm:betting} also holds ``point-wise'' for each realized outcome $y \sim \mbf{p}^*$. Indeed, this corresponds to the special case of $\mbf{p}^* = \one_y$.  These point-wise equalities will allow us to generalize the guarantee of proper betting to empirical scores as discussed in \cref{subsec:extent-empirical}. 

\begin{remark}
The key insight from \cref{thm:betting}  is to realize that $\mbf{s}_G(\mbf{p},\mbf{q}) := \nab G(\mbf{p}) - \nab G (\mbf{q})$ is the ``proper'' strategy to use and its profit admits a clean characterization. The decomposition of \cref{lem:decomposition} formalizes a natural intuition: a profitable forecast is not merely about being accurate (i.e., having big score gap), but is about sufficiently different from the market yet still remains sufficiently accurate.
Conversely, positive profit does not imply accuracy: it can be attributed entirely to the Bregman bonus. This also explains  the paradox in \cref{ex:inaccurate-win}: when a forecaster's prediction is far from the market price $\mbf{q}$, a large Bregman divergence can outweigh a negative score gap. 
\end{remark}

\begin{remark}[Beyond simplexes]
Notably, \cref{thm:betting} holds even when $\mbf{q} \not \in \Delta_K$. As an interesting application, we can also apply it when  $Q: = \sum_{k=1}^K q_k < 1$, i.e., there is trivial arbitrage opportunity by buying $\mbf{s} = \one $ at cost $Q (< 1)$ with a deterministic return  $1$  (assuming little liquidity loss). \cref{thm:betting}   captures this situation as a special case. By choosing $G(\mbf{q}) = (\sum_{k=1}^K q_k)^2$, we can verify that the proper betting $\mbf{s}_G = 2 [\sum_{k=1}^K (p_k - q_k)] \cdot  \one$, leading to deterministic profit $2(1 - Q)^2$ for any $\mbf{p}^*$. It can be further verified that both Bregman divergence and score gap term equal $(1 - Q)^2$ in this case, summing up to the $2(1 - Q)^2$ profit.  
\end{remark}

\subsection{Proper betting is (essentially) the only robustly profitable strategy} 

\cref{thm:betting} establishes the robust profitability of proper betting in a market with low liquidity loss. Our second main result  further shows that it is the only robustly profitable strategy, essentially --- the reason of ``essentially'' is to rule out trivial transformations such as rescaling the strategy (i.e., $\lambda \bm{s}_G(\mbf{p},\mbf{q})$) or adding a constant shift (i.e., $\bm{s}_G(\mbf{p},\mbf{q}) + \lambda \one$),    which will not change the profitability nature of the strategy.        

\begin{definition}[Essentially different betting]\label{def:essential-diff} 
We say a betting strategy $\bm{s}(\mbf{p},\mbf{q})$ is \emph{essentially different} from another strategy $\bm{s}'(\mbf{p},\mbf{q})$ if, after any rescaling of $\bm{s}(\mbf{p},\mbf{q})$ via mapping $\bm{s} \to \lambda \bm{s} $ and constant shifting via $\bm{s} \to  \bm{s} + \lambda \one$,  there always exists an interior market price $\mbf{q} \in \te{int}(\Delta_K)$ such that $||\bm{s}(\mbf{p}_t,\mbf{q}) - \bm{s}'(\mbf{p}_t,\mbf{q}) || \geq \epsilon$ for some forecast sequence $\{ \mbf{p}_t \}_{t=1}^\infty$ converging to $\mbf{q}$.

If two strategies are not essentially different, we  say they are \emph{essentially the same}.
\end{definition} 


 It is natural that  rescaling and constant shifting should not lead to an intrinsically different betting strategy: rescaling only proportionally changes betting sizes, whereas placing the all-one bet $\one$ costs $\one \cdot \mbf{q} = 1$ and pays off $1$ deterministically. The only non-trivial requirement for essentially different betting strategies is that they need to have non-negligible difference (i.e.,  at least $\epsilon$ difference) for some sequence of forecasts  $\{ \mbf{p}_t \}_{t=1}^\infty$ converging to $\mbf{q}$. This is a natural (and also necessary) requirement because otherwise two robustly profitable betting strategies can be superficially different. For instance, given any robustly profitable strategy, e.g.,
our proper betting $\bm{s}_G(\mbf{p},\mbf{q})$, one can construct $\bm{s}(\mbf{p},\mbf{q}) = \bm{s}_G(\mbf{p},\mbf{q}) + \epsilon \,    \mbf{e}_1 $  simply by additionally buying $\epsilon$ amount  of the first contract for some  $\epsilon $   smaller than the profit of $\bm{s}_G(\mbf{p},\mbf{q})$, which is at least the Bregman divergence $D_G(\mbf{q}, \mbf{p})$ as we will show.  This negligible modification can be done even for infinitely many pairs of $\mbf{p},\mbf{q}$ so long as they are all separated apart by at least a constant distance. Thus the real  test of   difference is at the limit when $\mbf{p} \to \mbf{q}$.    \cref{def:essential-diff}  is introduced precisely for that:  two strategies  $\bm{s}, \bm{s}'$ are  essentially different if  their  difference  $\bm{s}(\mbf{p}_t,\mbf{q}) - \bm{s}'(\mbf{p}_t,\mbf{q})$ does not converge to $0$  --   at least for some market price $\mbf{q}$ and  some sequence of forecasts $\{ \mbf{p}_t \}_{t=1}^\infty$ converging to  $\mbf{q}$.

Our second main theorem is  stated below. 

\begin{theorem} 
\label{thm:betting-unique}
Any robustly profitable betting strategy is essentially the same as the proper betting  $\mbf{s}_G(\mbf{p},\mbf{q})$. 
Formally, for any proper scoring rule $S$ and any betting strategy $\mbf{s}(\mbf{p}, \mbf{q})$ that is essentially different from  $\mbf{s}_G(\mbf{p}, \mbf{q})$, there exists $\mbf{p}, \mbf{p}^*, \mbf{q}$ such that  $S(\mbf{p}, \mbf{p}^*) > S(\mbf{q}, \mbf{p}^*)$ but the expected profit of  $\mbf{s}(\mbf{p}, \mbf{q})$   is strictly negative. This holds even for frictionless markets, i.e., $L_\rho = 0$.  
\end{theorem}


\begin{proof}[Proof sketch of \cref{thm:betting-unique}]
The proof is somewhat technical. We defer full proof to \cref{subsec:proof-main-thm} and only overview the proof ideas here. 
First, observe that it suffices to prove the theorem for frictionless markets since if any betting strategy is robustly profitable, then it must remain profitable in a frictionless market simply due to lower cost. Our proof thus argues that any robustly profitable betting strategy in a frictionless market is essentially the same as proper betting.  Given a  betting strategy $\mbf{s}(\mbf{p}_t, \mbf{q})$ that is essentially different from proper betting $\mbf{s}_G(\mbf{p}_t, \mbf{q})$, our proof here features an explicit construction of an infinite sequence of tuple $(\mbf{p}_t, \mbf{p}^*_t, \mbf{q})$ and then  leverages convergence properties to prove the existence of  tuples in this sequences that satisfies  $S(\mbf{p}_t, \mbf{p}^*_t) - S(\mbf{q}, \mbf{p}^*_t) > 0$ but the expected profit of  $\mbf{s}(\mbf{p}_t, \mbf{q})$ under the constructed $\mbf{p}^*_t$  is strictly negative.  
\end{proof}
 


A simple special case of \cref{thm:betting} is for the $K=2$ case, i.e., binary event betting. Note that  rescaling   via mapping $\bm{s} \to \lambda \bm{s} $ and constant shifting via $\bm{s} \to  \bm{s} + \lambda \one$ will not change a betting strategy $\bm{s}$'s robust profitability. Then any  betting strategy $\mbf{s}$, normalized to have $||\bm{s}|| = 1$ and $\bm{s} \cdot \one = 0$, is either  $s^+ = (1/\sqrt{2}, - 1/\sqrt{2})^\top$ (i.e., buying $\sqrt{2}$ of \emph{YES}) or $s^- = (-1/\sqrt{2}, 1/\sqrt{2})^\top$ (i.e., buying  $\sqrt{2}$ of \emph{NO}). If betting strategy $\bm{s}$ is essentially different from proper betting $\bm{s}_G$, then there exists $\bm{q}$ and a sequence $\{ \bm{p}_t \}_{t=1}^\infty$ converging to $\bm{q}$ such that the \emph{normalized} $\mbf{s}(\mbf{p}_t,\mbf{q}) ,  \mbf{s}_G(\mbf{p}_t,\mbf{q})$  satisfy $||\mbf{s}(\mbf{p}_t,\mbf{q}) -  \mbf{s}_G(\mbf{p}_t,\mbf{q})|| \geq \epsilon$ for some constant $\epsilon$. This implies that $\mbf{s}(\mbf{p}_t,\mbf{q}) ,  \mbf{s}_G(\mbf{p}_t,\mbf{q})$ must equal $s^+, s^-$ respectively. Therefore, if $\mbf{s}_G(\mbf{p}_t,\mbf{q})$ is profitable,  its opposite purchase behavior $\mbf{s}(\mbf{p}_t,\mbf{q})$ cannot be profitable, aligning with the theorem's statement. While this case with $K = 2$ is easy to see, the case with $K > 2$ becomes much more non-trivial since there are infinitely many normalized betting strategies. The formal proof in \cref{subsec:proof-main-thm} needs to leverage the geometry of the normalized strategy space.

\subsection{Implications and real market considerations} \label{sec:extensions} 

\subsubsection{A new characterization of proper scoring rules via betting profitability.}  
A keen reader might observe that the proper betting strategy $\mbf{s}_G(\mbf{p},\mbf{q})$ is well-defined for any $G$, not only those $G$ that are convex and correspond to proper scoring rules.  Is properness of the scoring rule  necessary to guarantee the robust profitability of $\mbf{s}_G(\mbf{p},\mbf{q})$? Our next proposition show that properness of the scoring rule is indeed necessary, in a very strong sense. That is, if $\mbf{p}, \mbf{q}$ are compared by some scoring rule $S$ satisfying $S(\mbf{p};\mbf{p^*}) > S(\mbf{q};\mbf{p^*})$, the properness of $S$ is necessary to guarantee the existence of a robustly profitable betting strategy (and when it exists, proper betting is one such strategy).   

\begin{proposition}[Equivalence between score properness and robust profitability] \label{thm:proper}
A scoring rule $S(\mbf{p}, y)$ is strictly proper \emph{if and only if} there exists a robustly profitable betting strategy $\mbf{s}(\mbf{p},\mbf{q})$ that guarantees positive profit in a frictionless market (i.e., $L_\rho = 0$)   for any $\bm{p}^* $ and $\mbf{p} \not =  \mbf{q}$ satisfying $S(\mbf{p};\mbf{p^*}) \geq S(\mbf{q};\mbf{p^*})$.   
\end{proposition}
 
\cref{thm:proper}, proven in \cref{a:characterization}, thus yields a novel characterization of proper scoring rules as precisely  those rules that yield robustly profitable betting strategy  when a forecast $\mbf{p}$ outperforms market $\mbf{q}$ under the scoring rule.  

\subsubsection{Proper betting is a strict generalization of canonical betting in AMMs}   

All our results thus far are applicable to any prediction market. Here we illustrate what proper betting strategy corresponds to in the special case of \emph{automated market makers} (AMMs)   and show that it is a strict generalization of the canonical betting strategy in AMMs.

An AMM is  governed by a strictly convex cost function $C(\mbf{x})$  such that  the current market price $\mbf{q} \in \Delta_K$ and net number of shares of each contract $\mbf{x}\in \bR^K$ always satisfies $\nab C(\mbf{x}) = \mbf{q}$ \citep{hanson2007logarithmic,chen2007utility}. A key property of AMMs is that if a forecaster's belief is some $\mbf{p}$, then this forecaster maximizes their expected profit by purchasing shares $\mbf{s}_{\mbf{q}\to\mbf{p}}$  that moves the market price  from $\mbf{q}$ to $\mbf{p}$, i.e.,   $\nab C(\mbf{x} +  \mbf{s}_{\mbf{q}\to\mbf{p}}) = \mbf{p}$ (see   \cref{a:impl} for more details). Interestingly, we observe that the above canonical betting strategy in AMMs corresponds precisely to the proper betting strategy of \cref{def:proper-bet} under the corresponding market scoring rule, or formally stated below  
\begin{fact}
  In automated market makers, we have $   \mbf{s}_{\mbf{q}\to\mbf{p}} = \mbf{s}_G(\mbf{p},\mbf{q}).$ 
\end{fact}

To see this, we recall that the convex conjugate of the cost function $C(\mbf{x})$ of an AMM, $G(\mbf{p}) = \sup_{\mbf{x}\in \bR^K}\bc{\mbf{p}\cd \mbf{x} - C(\mbf{x})}$, is precisely the convex potential that induces the market scoring rule $S(\mbf{p}, y) = G(\mbf{p}) + \nab G(\mbf{p})\cd (\one_y - \mbf{p})$. A basic fact about convex conjugates is that the inverse of the gradient function  $\nabla G: \Delta_K \to \bR^K$ is precisely the  $\nabla C (\mbf{x})$ function. Observing $ \mbf{x} + \mbf{s}_{\mbf{q}\to\mbf{p}}  = (\nabla C)^{-1}(\mbf{p})$, we have the betting strategy 
\begin{equation}\label{eq:amm-proper-equal}
    \mbf{s}_{\mbf{q}\to\mbf{p}}  =  \mbf{x} + \mbf{s}_{\mbf{q}\to\mbf{p}}  - \mbf{x} = (\nabla C)^{-1}(\mbf{p}) - (\nabla C)^{-1}(\mbf{q}) = \nabla G(\mbf{p}) - \nabla G (\mbf{q}) = \mbf{s}_G(\mbf{p},\mbf{q}).  
\end{equation}

A basic fact in AMMs is that the profit $\pi( \mbf{p}, \mbf{q})$ for moving the market from current price $\mbf{q}$ to new price $\mbf{p}$ is $S(\mbf{p}, \mbf{p}^*) - S(\mbf{q}, \mbf{p}^*)$ for ground truth distribution $\mbf{p}^*$ \citep{hanson2007logarithmic}, which corresponds precisely to the first term of \cref{thm:betting}. This is not  a coincidence since, as it turns out, the divergence term and liquidity loss term in \cref{thm:betting} cancel out precisely in any AMM, as stated below (proved in \cref{append:proof-amm-compare}).



\begin{corollary}[Instantiation of \cref{thm:betting} in AMMs]\label{prop:amm-compare}
Consider any AMM governed by a strictly convex cost function $C(\mbf{x})$ and let $S$ denote its corresponding market scoring rule determined by $C$'s convex conjugate $G(\mbf{p})$. Then we have $D_G(\mbf{q}, \mbf{p})  =  L_\rho (\mbf{s}^*;\mbf{q}) $ for the proper betting strategy $\mbf{s}^* = \nabla G(\mbf{p}) -  \nabla G(\mbf{q})$. Hence the profit of $\mbf{s}^*$ is  $S(\mbf{p}, \mbf{p}^*) - S(\mbf{q}, \mbf{p}^*)$.    
\end{corollary}

 \begin{remark}[Additional justification for the canonical betting strategy $\mbf{s}_{\mbf{q}\to\mbf{p}}$] \label{rm:theory-amm} 
The foundational result of \citet{hanson2003combinatorial,hanson2007logarithmic} already implies that the betting strategy  $\mbf{s}_{\mbf{q}\to\mbf{p}} (=\mbf{s}_G(\mbf{p},\mbf{q}))$ that   moves the market price  from $\mbf{q}$ to $\mbf{p}$  in AMMs is robustly profitable in the sense that it  guarantees positive   profit whenever the forecast $\mbf{p}$ beats the market under the  ground truth $\mbf{p}^*$.  Our \cref{thm:betting-unique} provides an additional justification and shows that this is essentially the only robustly profitable strategy. 
\end{remark}

\subsubsection{Extensions to empirical scores}\label{subsec:extent-empirical}
In practice the ground truth $\mbf{p}^*$ is unobservable, therefore a forecaster's  advantage over the market on a single event in expectation can never be validated. What is validatable however is the advantage on empirical scores. A  practical question thus is whether proper betting can convert empirical score advantages to realized profit. The answer turns out to be \emph{Yes}, inherently due to the fact that \cref{thm:betting}  holds also for realized outcome $y$ and can then be averaged over a sequence of outcome realizations. 

Suppose we only see realized outcomes $y_1, \ldots, y_n$ across a sequence of $n$ events with predictions $\mbf{p}_1, \ldots, \mbf{p}_n$ and market prices $\mbf{q}_1, \ldots, \mbf{q}_n$. Let the empirical forecast and market scores be $$\hat{S}_F := \frac{1}{n}\sum_i S(\mbf{p}_i, y_i)  \quad \text{ and }  \quad  \hat{S}_M := \frac{1}{n}\sum_i S(\mbf{q}_i, y_i).$$   \cref{thm:betting} can be extended from a single-event  to an empirical version below. 

\begin{corollary}[Empirical version of \cref{thm:betting}]\label{cor:empirical}
The sequential proper bet $\mbf{s}_i := \mbf{s}_G(\mbf{p}_i, \mbf{q}_i)$ on a realized sequence of outcomes $y_1, \ldots, y_n$ yields the following realized average profit: 
$$ \frac{1}{n} \sum_{i=1}^n  \mbf{s}_i \cdot (\one_{y_i} - \mbf{q}_i)   \;=\;  \hat{S}_F - \hat{S}_M  \;+\;  \frac{1}{n}  \sum_{i=1}^n D_G(\mbf{q}_i, \mbf{p}_i) - \frac{1}{n} \sum_{i=1}^n L_{\rho}(\mbf{s}_i;\mbf{q}_i). $$  
\end{corollary} 
Hence the sequential proper betting has positive profit if the  gap $ \hat{S}_F - \hat{S}_M $ of empirical scores and average divergence   $\frac{1}{n} \sum_{i=1}^n D_G(\mbf{q}_i, \mbf{p}_i)$ together offset the average liquidity loss $\frac{1}{n} \sum_{i=1}^n L_{\rho}(\mbf{s}_i;\mbf{q}_i)$. In \cref{sec:experiments}, we experimentally observe a few forecasting agents profiting due to divergence from the market despite below-market accuracy.

\begin{remark}[Bet sizes should vary across events]
A natural intuition is that a forecaster whose accuracy is measured with equal weight across events should place equal bet size on these events. \cref{cor:empirical} says otherwise: the correct bet sizes  are proportional to $\mbf{s}_i = \nab G(\mbf{p}_i) - \nab G(\mbf{q}_i)$, with larger bets on events where the forecaster's prediction deviates more from the market.
\end{remark}



\subsubsection{Extension to non-zero bid-ask spreads and correlated market outcomes.}\label{sec:bid-ask-spread} In real prediction markets, factors such as low liquidity or platform transaction fees imply a positive difference between the lowest price at which someone will sell and the highest price at which someone will buy. For the $k$-th outcome of an event, a real market has $q^+_k + q^-_k \ge 1$, where ${q}^+_k$ denotes the market price for buying the outcome (bid price) and ${q}^{-}_k$ denotes the market price for selling the outcome (ask price). \cref{thm:betting} is not directly applicable to this situation because when any $k$'th entry of the proper betting $\nabla G(\mbf{p}) - \nabla G(\mbf{q}) $ is negative, its execution requires buying \emph{NOT $k$} at price $1-q^+_k$, which is not possible here since the price of \emph{NOT $k$} is ${q}^{-}_k (> 1-q^+_k)$. 

The key idea to extend \cref{thm:betting} to such real-world settings with  $q^+_k + q^-_k > 1$   is to decompose each outcome $k\in [K]$ into two \emph{binary outcome} events, \emph{YES $k$} and \emph{NOT $k$}, that have prices equal to ${q}^+_k$ and ${q}^-_k$ respectively. It can be shown that a proper betting will buy   \emph{YES $k$} contracts  when $p_k >  {q}^+_k$, buy    \emph{NO $k$} contracts  when $p_k < 1-   {q}^-_k$, and not act when $1-   {q}^-_k \leq p_k \leq   {q}^+_k$. Moreover, this can be further extended to situations when   the  $K$ outcomes are not completely disjoint. For instance, a popular event on prediction markets is to predict bitcoin price where the outcome could be ``price at least \$7K'', ``price at least \$8K'', etc. In such cases, we treat each  single outcome as a separate binary market. The full details of the reductions can be found in \cref{a:spread}.

\subsubsection{Extension to sequential betting.} Events in prediction markets often take days or weeks to resolve, during which both the forecaster's prediction $\mbf{p}^t$ and the market price $\mbf{q}^t$ evolve in response to the revelation of new information. \cref{sec:portfolio} extends \cref{thm:betting} to this dynamic setting via two strategies: a \emph{fundamental}-driven strategy that executes the proper bet $\mbf{s}_G(\mbf{p}^t, \mbf{q}^t)$ at each round and holds positions to resolution, and a \emph{momentum}-driven strategy that maintains $\mbf{s}_G(\mbf{p}^t, \mbf{q}^t)$ as a rebalanced position. Both inherit the profit decomposition of \cref{thm:betting} summed across rounds, and the comparison reveals what each strategy needs for its performance guarantee: a per-round accuracy edge against the unobservable ground truth $\mbf{p}^*$ (fundamental), versus a per-round accuracy edge against the next-step market price $\mbf{q}^{t+1}$ (momentum) --- an operationally observable condition.




\section{Experiments}  
\label{sec:experiments}

Testing betting strategies requires data that pairs forecasts with contemporaneous market prices and realized outcomes. Such data is hard to collect from human forecasters at scale, but recent AI forecasting benchmarks built on real-money prediction markets---most notably Prophet Arena \citep{yang2025llm}---provide a scalable alternative. We therefore use AI forecasts across thousands of archived markets to evaluate how different betting strategies translate predictive accuracy into realized returns. The evaluation has two stages: first, an offline study on the historical benchmark, where we simulate each market under zero price impact; second, a live pilot test in real prediction markets, exposing our strategy to the full set of real-market frictions.  

\subsection{How good are proper betting strategies?}
\label{sec:compare}
We use forecasts made by 18 frontier models, each with predictions on at least 600 resolved Kalshi and Polymarket events (up to approximately 3,900 events for the most heavily covered models), for a total of 94,453 forecasts\footnote{Many forecasts may be made per event at different times} collected in collaboration with Sooth Labs, spanning a variety of domains such as sports, politics, economics, and crypto. Letters appended to model names denote distinct prompt variants of the same underlying model.\footnote{Full betting data is released \href{https://prophets-dataset.onrender.com/}{\textcolor{green!50!black}{\textbf{here.}}}} Each market includes a forecast, contemporaneous market price, and realized outcome (YES/NO). For each model and betting strategy, we apply the strategy to historical forecast–market (bid and ask) price pairs, simulate the resulting trades, and compute the realized return after market resolution, allowing us to compare proper betting strategies directly against standard heuristic baselines.

Specifically, each betting strategy is to determine a set of weights --- how to allocate a fixed budget and on what direction --- over a collection of prediction markets with forecasts and market prices. We implement the three proper betting strategies according to \cref{tab:scoring-rules}.
We also consider a few baseline strategies derived from a few intuitive heuristics: (i) \textbf{Max-Margin}, which concentrates capital on the largest disagreements with the market, either across all markets or by selecting the single largest-margin bet per event group; (ii) \textbf{Inverse-Margin}, which down-weights large disagreements under the assumption that extreme deviations are less reliable; (iii) \textbf{Kelly-Alike}, which scales exposure according to Kelly-style edge-to-odds ratios\footnote{In empirical tests, we determine the Kelly criterion performs particularly poorly because it is highly sensitive to miscalibrated probabilities. As a result, it becomes highly unstable under model error.}; and (iv) \textbf{Kelly Criterion} \citep{kelly1956new}, which maximizes asymptotic log-wealth growth. Appendix~\ref{app:appCanonical} expands on the formal setup and implementation details of the betting strategies.

 \cref{tab:betting-strategy-std} reveals why proper betting strategies warrant study: heuristic allocations fail to yield consistent positive returns, and for stronger models, Brier-weighted allocation is the only strategy that reliably produces positive ROI. Inverse-Margin limits downside for weaker models by concentrating capital on near-market (i.e., minimal risk) bets, but as a consequence of its design, also limits opportunities for creating meaningful upside.
More generally, these failures of heuristic betting strategies show that predictive accuracy alone is insufficient to create profit. How probabilities are translated into position sizes is a key determinant of returns, motivating the study of proper betting strategies. Appendix~\ref{app:appCanonical} reports the results for all evaluated models.

\begin{table*}[t]
\centering
\scriptsize
\setlength{\tabcolsep}{3pt}
\renewcommand{\arraystretch}{1.05}

\caption{Pairwise comparison of the Proper Betting (Brier) strategy against each of the four baseline heuristics for selected top forecasting agents, with each column reporting the difference between Proper Betting and the corresponding heuristic. Confidence intervals are paired, event-clustered 95\% bootstrap intervals; bold values highlight the differences whose CI excludes zero.}
\label{tab:betting-strategy-std}

\newcommand{\ci}[2]{\text{\scriptsize [#1, #2]}}
\resizebox{\textwidth}{!}{%
\begin{tabular}{@{}lcccccc@{}}
\toprule
Model & $\Delta S$ & $-$ Max (Mkt) & $-$ Max (Grp) & $-$ Inv-Margin & $-$ Kelly-Alike & $-$ Kelly \\
\midrule

Gemini 3.1 Pro (B)
& \textbf{+0.0076} \ci{+0.0014}{+0.0136}
& \textbf{+30.3} \ci{+20.7}{+39.9}
& \textbf{+30.1} \ci{+15.5}{+42.9}
& \textbf{+39.2} \ci{+27.7}{+50.6}
& \textbf{+24.3} \ci{+16.1}{+32.8}
& \textbf{+138.9} \ci{+127.3}{+150.4} \\

Claude Fable 5
& \textbf{+0.0039} \ci{+0.0008}{+0.0073}
& \textbf{+26.8} \ci{+15.3}{+40.0}
& \textbf{+48.3} \ci{+34.2}{+63.1}
& \textbf{+33.6} \ci{+20.0}{+48.9}
& \textbf{+16.8} \ci{+7.3}{+27.4}
& +129.0 \ci{$-1.3\times10^{9}$}{+141.8} \\

GPT-5.6 Sol
& \textbf{+0.0033} \ci{+0.0003}{+0.0067}
& \textbf{+22.4} \ci{+12.2}{+35.4}
& \textbf{+39.9} \ci{+24.3}{+55.7}
& \textbf{+28.5} \ci{+15.7}{+43.7}
& \textbf{+15.7} \ci{+6.0}{+27.1}
& -1220.4 \ci{$-2.2\times10^{8}$}{+122.5} \\

GPT-5.5 (B)
& -0.0059 \ci{-0.0129}{+0.0009}
& \textbf{+13.0} \ci{-1.2}{+29.3}
& \textbf{+20.3} \ci{-0.7}{+41.1}
& \textbf{+18.6} \ci{+0.1}{+38.9}
& +8.6 \ci{-2.7}{+21.4}
& \textbf{+116.8} \ci{+99.8}{+136.0} \\

GPT-5.5 (C)
& -0.0056 \ci{-0.0101}{-0.0012}
& +8.8 \ci{-1.5}{+18.7}
& +3.6 \ci{-28.4}{+27.3}
& +7.0 \ci{-5.5}{+19.4}
& +2.0 \ci{-6.3}{+10.0}
& \textbf{+107.3} \ci{+94.9}{+119.2} \\

\bottomrule
\end{tabular}
}
\end{table*}





%
\subsection{How to choose proper betting strategies for different forecaster personas?}

\paragraph{Score gap vs. divergence.}


\cref{tab:scoring-decomposition} empirically illustrates the decomposition in \cref{lem:decomposition}. We report the proper score gap and divergence terms summed across bets and normalized by total spend, which exactly matches with the per-bet decomposition in Equation~\ref{lem:decomposition} in ROI units. This highlights how proper betting strategies induce trade-offs between accuracy and divergence. Profitability depends not only on the score gap ($\Delta S$), but also on the divergence term ($D$), so models with worse accuracy can still achieve higher ROI if they generate sufficiently large divergence. For example, under the Log rule, Gemini 3.1 Pro (B) achieves higher ROI than GPT 5.6 Sol despite a worse score gap. Moreover, a model’s score under a given scoring rule does not by itself determine which rule yields the best returns, since divergence is rule-dependent. For instance, for GPT 5.5 (B), Brier yields a worse score gap than Log ($-15.4$ vs.\ $-9.5$) but higher ROI ($+16.8\%$ vs.\ $+11.7\%$), driven by a substantially larger divergence term ($32.2$ vs. $21.2$).
As such, identifying the most profitable models and betting strategies is non-trivial, as profitability depends not only on accuracy but also on how those strategies factor in disagreement with the market. Appendix~\ref{app:addModelExps} reports the decomposition results for all models.

\newcommand{\est}[2]{#1 {\scriptsize #2}}
\newcommand{\best}[2]{\textbf{#1} {\scriptsize \textbf{#2}}}
\newcommand{\broi}[2]{\textbf{\textit{#1}} {\scriptsize \textbf{\textit{#2}}}}

\begin{table*}[t]
\centering
\small
\setlength{\tabcolsep}{4pt}
\renewcommand{\arraystretch}{1.15}
\caption{\textbf{ROI decomposition on the expanded dataset}, where
$\mathrm{ROI}=\Delta S + D$ for each model and proper scoring rule.
Bold denotes significantly positive score gaps (95\% CI above zero),
and bold italic denotes significantly positive ROI.}
\label{tab:scoring-decomposition}
\resizebox{\textwidth}{!}{%
\begin{tabular}{lrrrrrrrrrr}
\toprule
Model
& Brier $\Delta S$
& Brier $D$
& Brier ROI
& Log $\Delta S$
& Log $D$
& Log ROI
& Sph $\Delta S$
& Sph $D$
& Sph ROI
& Events \\
\midrule

Gemini 3.1 Pro (B)
& \best{+11.7}{[+2.4, +21.1]}
& \est{+27.2}{[+22.7, +32.1]}
& \broi{+38.9}{[+27.4, +50.9]}
& \best{+7.5}{[+0.8, +14.0]}
& \est{+20.9}{[+18.2, +23.8]}
& \broi{+28.4}{[+20.4, +36.2]}
& \best{+11.3}{[+3.5, +19.5]}
& \est{+21.9}{[+18.3, +25.8]}
& \broi{+33.2}{[+23.9, +42.9]}
& 1113 \\

Claude Fable 5
& \best{+12.3}{[+3.2, +22.0]}
& \est{+16.7}{[+11.8, +22.2]}
& \broi{+29.0}{[+16.7, +42.7]}
& \best{+8.6}{[+1.3, +15.6]}
& \est{+12.9}{[+9.8, +16.5]}
& \broi{+21.5}{[+12.2, +30.6]}
& \best{+10.2}{[+1.4, +18.6]}
& \est{+11.9}{[+8.5, +15.9]}
& \broi{+22.1}{[+11.6, +33.0]}
& 1140 \\

GPT-5.6 Sol
& \best{+10.9}{[+1.0, +21.7]}
& \est{+15.9}{[+12.1, +21.9]}
& \broi{+26.7}{[+14.4, +41.4]}
& \best{+10.7}{[+3.0, +18.5]}
& \est{+12.6}{[+10.1, +16.1]}
& \broi{+23.3}{[+14.5, +32.4]}
& \best{+12.1}{[+3.2, +21.4]}
& \est{+11.4}{[+8.5, +15.7]}
& \broi{+23.5}{[+13.0, +35.8]}
& 866 \\

GPT-5.5 (B)
& \est{-15.4}{[-29.6, -0.4]}
& \est{+32.2}{[+25.8, +39.5]}
& \broi{+16.8}{[+2.3, +33.3]}
& \est{-9.5}{[-20.7, +0.5]}
& \est{+21.2}{[+17.9, +24.9]}
& \broi{+11.7}{[+0.3, +22.2]}
& \est{-11.5}{[-23.0, -0.7]}
& \est{+24.4}{[+19.4, +30.1]}
& \broi{+12.9}{[+0.9, +25.0]}
& 842 \\

GPT-5.5 (C)
& \est{-15.6}{[-26.3, -5.6]}
& \est{+22.9}{[+18.6, +27.3]}
& \est{+7.3}{[-3.3, +17.3]}
& \est{-13.1}{[-21.4, -5.5]}
& \est{+18.4}{[+15.2, +21.6]}
& \est{+5.2}{[-3.8, +13.9]}
& \est{-14.6}{[-23.4, -6.2]}
& \est{+16.5}{[+13.4, +20.1]}
& \est{+1.9}{[-6.6, +10.7]}
& 979 \\

\bottomrule
\end{tabular}%
}
\end{table*}

\paragraph{Forecaster personas.}

As seen in \cref{tab:scoring-decomposition}, different scoring rules yield materially different returns. To identify when each rule is preferable, we construct a controlled experiment using four synthetic model ``personas'' that all achieve the same Brier score on the same set of real prediction markets, but differ in how their predictions deviate from market prices. We focus on the margin range $|p - q| \in [0, 0.5]$, as higher-margin forecasts are too sparse in empirical models to support reliable or broadly applicable analysis; although rare, such high-margin bets can occasionally be correct and highly profitable, which we examine further in \cref{app:profitabilityMargins}. We run the experiment in two regimes—one where the model beats the market and one where it loses—to examine how accuracy interacts with the choice of proper betting strategy and the resulting implications.



\paragraph{Persona design and motivation.}
We use 27,516 markets across 3,511 questions collected from the Prophet Arena pipeline between August 1, 2025 and April 15, 2026. Models differ in how aggressively they deviate from market prices, and in whether those deviations tend to pay off. To capture this variation, we construct personas defined by two features: (i) the proportion of small-margin bets (i.e., $\lvert p - q \rvert \leq 0.15$), and (ii) the accuracy drop-off across margins, measured by the slope of the directional win rate (the fraction of bets whose chosen side matches the realized outcome) as a function of margin size $\lvert p - q \rvert$. Both regimes are calibrated to a fixed $\pm 0.05$ relative quadratic scoring rule gap. \cref{app:syntheticMethod} outlines the details of the synthetic forecast generation process.


\begin{figure}
    \centering

    \begin{subfigure}[t]{0.42\linewidth}
        \centering
        \includegraphics[width=\linewidth]{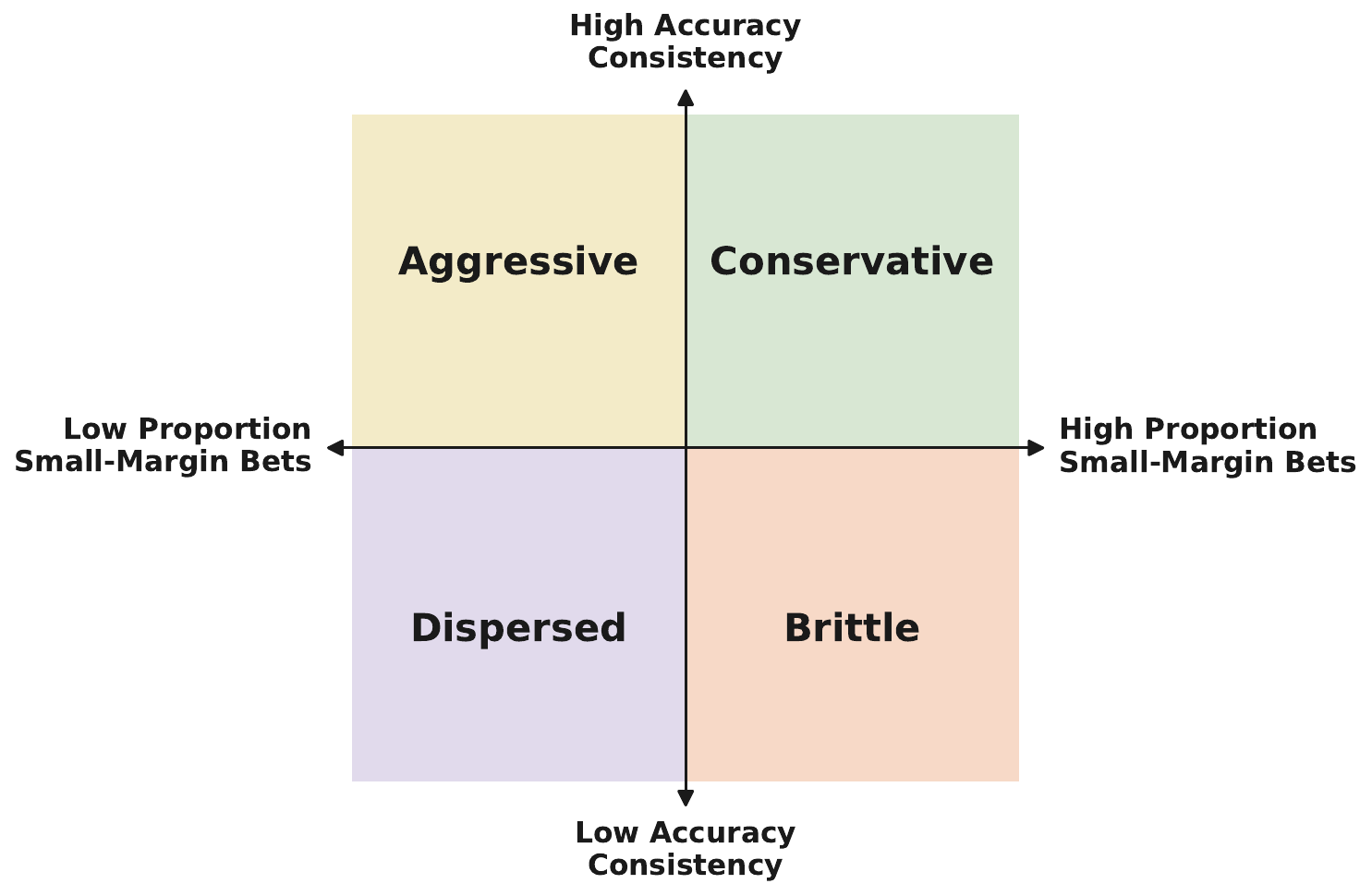}
        \caption{Taxonomy of LLM forecasting personas by proportion of forecasts at small margins and accuracy consistency across margins. }
        \label{fig:personaMap}
    \end{subfigure}
    \hfill
    \begin{subfigure}[t]{0.45\linewidth}
        \centering
        \includegraphics[width=\linewidth]{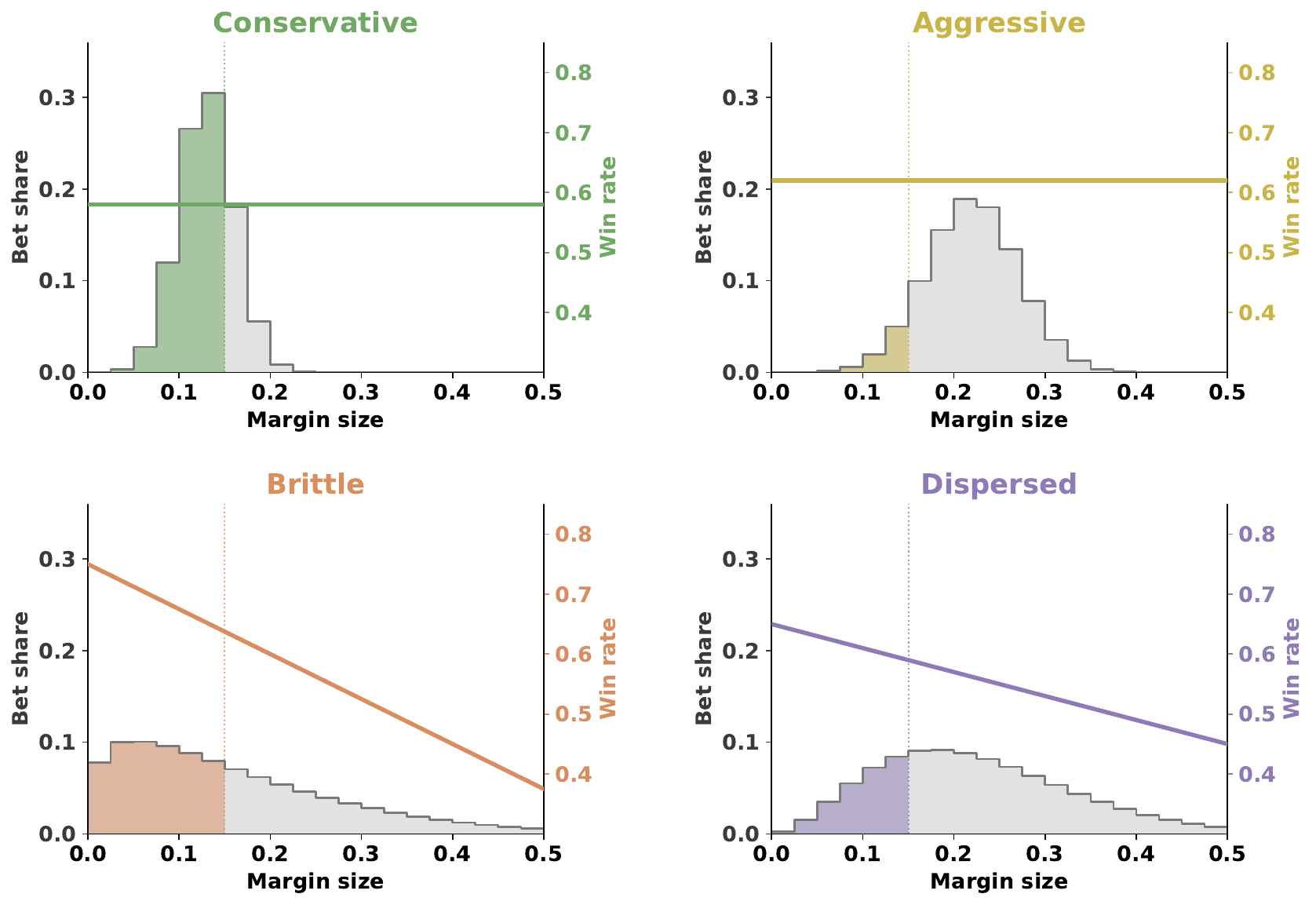}
        \caption{Each panel shows a persona’s margin distribution (grey, left axis) and win-rate at different margins (color, right axis). 
        }
        \label{fig:personaArchetypes}
    \end{subfigure}

    \caption{Schematics for synthetic forecasting personas and their prediction profiles.}
    \label{fig:personaCombined}
\end{figure}

We define four personas spanning representative combinations of these dimensions, as shown in \cref{fig:personaCombined}. \textbf{Conservative} models concentrate most forecasts in the small-margin region and exhibit flat accuracy across margins. \textbf{Aggressive} models place fewer small-margin bets, with accuracy remaining flat across margins. \textbf{Dispersed} models spread bets broadly across margin sizes and show a gradual decline in win-rate as margin increases. \textbf{Brittle} models have a significant proportion of small-margin forecasts but exhibit a sharply declining win-rate as margin increases. For further details, see Appendix~\ref{app:syntheticMethod}.


\begin{table*}[ht]
  \centering
  \caption{Simulated persona ROI (\%) under proper betting strategies. S-M Share denotes the share of forecasts with small margin: $|p - q| \leq 0.15$. Consistency is the OLS slope of per-bet directional win rate against margin size.}
  \label{tab:persona-regimes}

  \small
  \setlength{\tabcolsep}{4pt}
  \renewcommand{\arraystretch}{1.05}

  \begin{tabular}{@{}lcc|ccc|ccc@{}}
    \toprule
    & & & \multicolumn{3}{c}{\textit{Regime A (Outperforms)}}
    & \multicolumn{3}{c}{\textit{Regime B (Underperforms)}} \\
    \cmidrule(lr){4-6} \cmidrule(lr){7-9}

    \textbf{Persona} & \textbf{S-M Share}
 & \textbf{Consistency}
    & \textbf{Brier} & \textbf{Log} & \textbf{Sph.}
    & \textbf{Brier} & \textbf{Log} & \textbf{Sph.} \\
    \midrule

  Conservative   & 0.85 & $0.0$ & \textbf{+48.8} & +46.2 & +44.8 & -76.8 & \textbf{-71.5} & -82.7 \\           
  Aggressive     & 0.58 & $0.0$ & \textbf{+37.7} & +37.4 & +31.3 & -38.8 & \textbf{-30.1} & -46.9 \\           
  Dispersed      & 0.68 & $-0.2$ & \textbf{+44.3} & +42.5 & +38.6 & -30.4 & \textbf{-20.7} & -38.1 \\           
  Brittle        & 0.77 & $-0.6$ & \textbf{+45.4} & +43.8 & +41.0 & -60.2 & \textbf{-52.2} & -67.9 \\ 
  
    \bottomrule
  \end{tabular}
\end{table*}
We evaluate each persona under two regimes: Regime A, where the model outperforms the market and all strategies yield positive ROI, and Regime B, where it underperforms and all strategies incur losses. A symmetric $\pm 0.05$ quadratic score gap (market quadratic score = 0.839) isolates the effect of persona and strategy while holding average accuracy constant.

As shown in \cref{tab:persona-regimes}, Brier is optimal in Regime A for all personas, while Log is optimal in Regime B. In Regime A, the positive Brier gap implies that expected returns are positive across margin bins, so linear scaling in $|p - q|$ is optimal. In Regime B, losses are dominated by high-margin errors, making Log’s sublinear weighting preferable as it attenuates exposure to large, costly mistakes. This effect is most pronounced for Conservative personas, which must incur a high error rate on small-margin bets to sustain a negative Brier gap, leading to broad losses. Additional analysis and weighting visualizations are provided in \cref{app:weight-surfaces}.



\begin{figure}
\centering

\begin{subfigure}{0.45\textwidth}
\centering
\footnotesize
\setlength{\tabcolsep}{2pt}
\renewcommand{\arraystretch}{1.1}

\setlength{\tabcolsep}{4pt}
\renewcommand{\arraystretch}{1.1}

\setlength{\tabcolsep}{4pt}
\renewcommand{\arraystretch}{1.1}

\begin{tabular}{@{}l c c c c@{}}
\toprule
\textbf{Model}
& \textbf{S-M Share}
& $\boldsymbol{\Delta S}$
& \textbf{ROI}
& \\ 
\midrule

GPT-5.2 (B)   & 0.82 & $-0.044$ & $-4.4$ & (L) \\
Grok 4.1      & 0.80 & $-0.053$ & $-9.1$ & (L) \\

\rowcolor{orange!7}
\textbf{\textit{Brittle}}
& \textbf{0.80} & $\mathbf{-0.054}$ & $\mathbf{-8.6}$ & \textbf{(L)} \\

\midrule

LLaMA 4       & 0.67 & $-0.115$ & $-4.9$ & (L) \\
DeepSeek R1   & 0.59 & $-0.140$ & $-5.4$ & (L) \\

\rowcolor{teal!7}
\textbf{\textit{Dispersed}}
& \textbf{0.64} & \textbf{-0.120} & \textbf{-7.2} & \textbf{(L)} \\

\midrule

Gemini 3      & 0.94 & +0.002        & +14.2 & (B) \\
Claude Opus   & 0.89 & +0.000$^{\ast}$ & +17.5 & (B) \\

\rowcolor{purple!7}
\textbf{\textit{Conservative}}
& \textbf{0.91} & \textbf{+0.001} & \textbf{+15.8} & \textbf{(B)} \\

\bottomrule
\end{tabular}

\caption{Persona classification across models. ROI reported is under the best proper scoring rule: Log (L) or Brier (B). $^{\ast}$ denotes a positive $\Delta S$ rounded to $+0.000$.}
\label{tab:personaMargins}
\end{subfigure}
\hfill
\begin{subfigure}{0.49\textwidth}
\centering
\includegraphics[width=\linewidth]{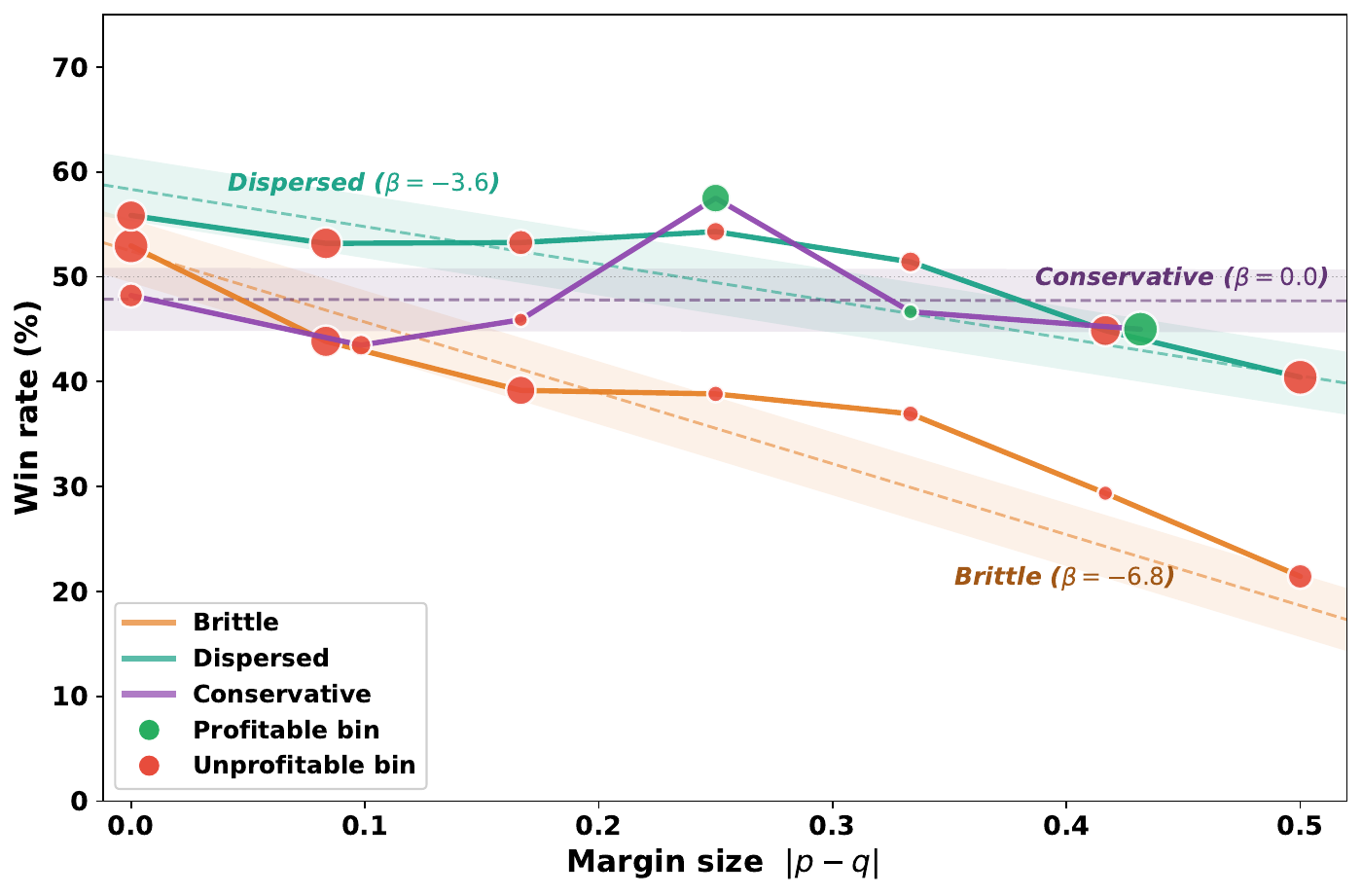}
\caption{Binned per-bet win rate vs.\ margin $|p-q|$. Each dot has area proportional to profit for their persona. Dashed lines are OLS fits where $\beta$ denotes the slope. Overlapping points are jittered horizontally for clarity.}
\label{fig:margin-winrate}
\end{subfigure}

\caption{Mapping real models to personas.}
\label{fig:real-profiles-and-winrate}

\end{figure}

\paragraph{Mapping real models to personas.}
\label{sec:personas}

To map synthetic personas to real model behavior, we evaluate each model along the two defining dimensions: small-margin forecasts and accuracy consistency. Because these patterns only emerge reliably at scale, we evaluate each model on its full event history rather than the standardized subset of events. \cref{fig:real-profiles-and-winrate} summarizes these results, with \cref{tab:personaMargins} showing models' proportion of small-margin forecasts and \cref{fig:margin-winrate} showing the win-rate profiles as a function of margin size. For classification of all evaluated models and additional analysis, see Appendix~\ref{app:mapPersonas}.

Along the margin dimension, models separate cleanly into the three persona groups: Brittle models concentrate most bets at small margins, Dispersed models spread forecasts more broadly across margins, and Conservative models place nearly all bets at small margins. \cref{fig:margin-winrate} further confirms the taxonomy along the accuracy-consistency dimension: Brittle models exhibit a steep decline in win rate as margins increase, Dispersed models decline more gradually, and Conservative models remain roughly flat across margins.
Taken together, these results show that personas provide a simple way to summarize how model behavior translates into returns. This abstraction lets us reason about profitability and select or adapt betting strategies based on these behavioral patterns.

\subsection{A live pilot test under active portfolio management}

We further evaluate proper betting in a live portfolio setting using the momentum-driven strategy described in \cref{sec:extensions} (with details deferred to Appendix~\ref{sec:portfolio}).  Notably, our main goal here is NOT to establish statistically reliable trading alpha or optimal risk-adjusted performance,  which is impossible for us given both our resource  and  legal constraints.  Instead, these experiments represent our attempt to go beyond pure simulation studies and \emph{test whether proper betting can survive  live execution} which --  unlike the previous backtests --  must contend with bid--ask spreads, limited liquidity, slippage, and discrete tick sizes. As a result, even accurate forecasts may not translate cleanly into realized returns, since trades must be executed at available prices rather than idealized market probabilities.

We deploy Gemini~3 using the Brier-derived betting strategy on Kalshi events for 26 days (April/May 2026) with an initial budget of \$200. Operating on a fixed two-hour cadence, the agent evaluates all open markets satisfying an ex-ante eligibility filter that excludes contracts with ambiguous resolution criteria and settings known to impair LLM forecasting performance (see \cref{sec:exclusionCriteria}).  Summary statistics and cumulative ROI from the live pilot test are shown in \cref{fig:live-results}, whereas the trading history \href{https://prophets-profit.onrender.com/}{\textcolor{green!50!black}{\textbf{is documented here}}}, which logs performance of the trading agent over the trading period.
Over the trading period, the agent executed 236 orders across 129 markets, achieving a $3$-trimmed ROI (i.e., dropping both best 3 and worst 3 trades) of \textbf{+24.2\%}, whereas the untrimed ROI is +80.33\% (a P\&L-weighted Brier decomposition attributes $\Delta S = +0.7205$ and $D = +0.0828$. The Sharpe ratio is \textbf{3.35}, with a 90\% confidence interval of $[0.9, 6.12]$. Note that we report the Sharpe ratio only for the live test, where the daily returns form a continuous time series amenable to risk-adjusted analysis; the previous offline experiments resolve each event independently and most models post negative returns, so a Sharpe ratio there would be uninformative. 

We document the full implementation, analysis of representative trades, and trading logs in \cref{app:liveTrading}.

\begin{figure}[t]
\centering

\begin{subfigure}[c]{0.42\linewidth}
\centering
\small
\begin{tabular}{lr}
\toprule
Starting capital & \$200.00 \\
Trading days & 26 \\
Forecasts issued & 1,605 \\
Trade signals placed (filled) & 396 (236) \\
Markets traded & 129 \\
Ending capital & \$360.67 \\
ROI & $+80.33\%$ \\
$3$-trimed ROI & +24.2\% \\
Sharpe ratio ($\sqrt{365}$) & $3.35$ \\
90\% CI of Sharpe &  $[0.90, 6.12]$ \\  
Shares transacted & 2{,}657 \\
Exchange fees paid & \$26.31 \\
Win rate & $50.9\%$ \\
maximum drawdown &  \$31.47 (15.7\%) \\  
\bottomrule
\end{tabular}
\caption{Live trading statistics for the Gemini~3 agent on Kalshi.}
\label{tab:live-trading}
\end{subfigure}
\hfill
\begin{subfigure}[c]{0.54\linewidth}

\centering
\includegraphics[width=\linewidth]{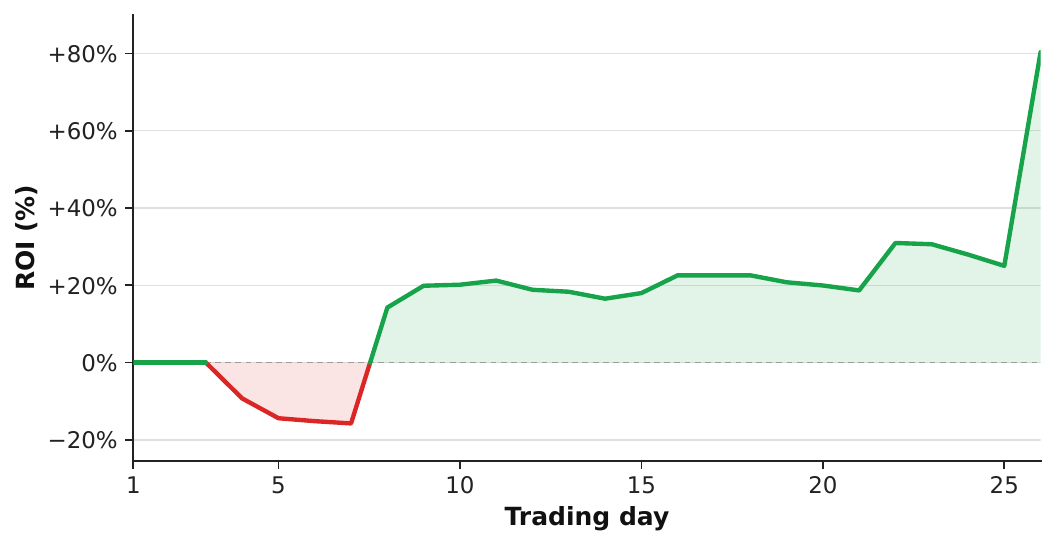}
\caption{ROI over trading days of the forecaster's live Kalshi trades over 26 trading days.}
\label{fig:live-pnl}
\end{subfigure}

\caption{Live pilot results for the Gemini~3 trading agent on Kalshi prediction markets.}
\label{fig:live-results}
\end{figure}

\section{Conclusion}

We investigate the puzzle that forecasters can outperform prediction markets in accuracy yet still lose money, by establishing a formal equivalence between predictive accuracy and  trading profitability. Key to our result is to identify a ``proper'' betting strategy for arbitrary prediction markets. Moreover, we demonstrate that this strategy  strictly generalizes the canonical betting strategy only applicable to the special AMM markets,   is essentially the only robustly profitable   strategy, and  has deep connection to proper scoring rules.     
Through this lens, it is unsurprising that proper betting emerges as the only strategy reliably converting accuracy into ROI across thousands of archived AI forecasts. While our live experiments in the real Kalshi markets is limited, due to both resource and legal constraints,  the month-long live pilot test with $+24.2\%$ (trimed) ROI and $3.35$ Sharpe ratio does hint on the promise of practical viability of our theory.  

Several future directions remain open. First, our persona analysis in \cref{sec:personas} hints at the potential of \emph{data-adaptive scoring-rule selection}: identifying the most suitable rule for each forecaster. Second, the framework opens \emph{market-design} questions on which scoring rules and market structures best elicit accurate forecasting from participants at scale. Finally, our guarantees concern \emph{expected} profit; extending them to risk-adjusted criteria would narrow the gap between theory and live deployment.

\bibliographystyle{plainnat}
\bibliography{references}

\appendix
\newpage 
\section{Implementations of prediction markets} \label{a:impl}

A prediction market needs an underlying matching mechanism that determines how buy and sell orders interact and at what prices; the choice shapes the liquidity, transaction costs, and incentives forecasters face. Two design families dominate practice---automated market makers and central limit order books---and both are recovered as special cases of the price-impact-function abstraction in \cref{sec:prelim} via specific choices of $\rho$ and $L_\rho$.

\paragraph{Automated market makers.}
A classic implementation of a prediction market is an \emph{automated market maker} (AMM) governed by a cost function \citep{hanson2003combinatorial,hanson2007logarithmic, chen2007utility, abernethy2013efficient}. The AMM maintains a vector $\mbf{x}\in \bR^K$ recording the net number of shares of each outcome it has sold so far, and commits to a strictly convex, differentiable \emph{cost function} $C: \bR^K \to \bR$. A forecaster who executes the position vector $\mbf{s}\in\bR^K$ pays the total cost $C(\mbf{x}+\mbf{s}) - C(\mbf{x})$ of moving the share state to $\mbf{x}+\mbf{s}$, and receives the payout $\mbf{s}\cd \one_y$ for a realized outcome $y$, with expected profit
$\pi^{\text{AMM}}(\mbf{s}, \mbf{p}^*) := \mbf{s}\cd \mbf{p}^* - \bp{C(\mbf{x}+\mbf{s}) - C(\mbf{x})}.$
The cost function $C$ is chosen so that it always satisfies $\nab C(\mbf{x}) \in \Delta_{[K]}$ for all $\mbf{x}$. The marginal cost $\mbf{q} = \nab C(\mbf{x})$  is  the \emph{spot price}. In the price-impact notation of \cref{sec:prelim}, an AMM corresponds to $\rho(\mbf{s}) = \nab C(\mbf{x}+\mbf{s})$ and $L_\rho(\mbf{s}; \mbf{q}) = D_C(\mbf{x}+\mbf{s},\mbf{x})$.

The key structural property of this construction is that the convex conjugate of $C$ on the simplex, $G(\mbf{p}) = \sup_{\mbf{x}\in \bR^K}\bc{\mbf{p}\cd \mbf{x} - C(\mbf{x})}$, is the convex potential that induces a strictly proper scoring rule $S(\mbf{p}, y) = G(\mbf{p}) + \nab G(\mbf{p})\cd (\one_y - \mbf{p})$. 

\paragraph{Central limit order books.}
The largest prediction markets today (e.g.\ Kalshi, Polymarket) instead adopt a \emph{central limit order book} (CLOB), the matching mechanism standard in equity and futures exchanges: opposing limit orders are matched directly between participants, and liquidity is supplied dynamically by the resting orders rather than by a designated maker. CLOBs are favored over AMMs in practice because the cost-function approach is hard to scale and operate: liquidity must be subsidized up front through a fixed parameter, making thousands of simultaneous markets infeasible; the cost function cannot be updated post-deployment, whereas CLOB makers adjust quotes to news and adverse selection in real time; and derivatives regulations---such as the CFTC's, under which Kalshi operates as a designated contract market---are written around order-book mechanics rather than algorithmic counterparties. Polymarket itself migrated from an AMM to a CLOB after launch \citep{polymarket2023clob, ng2026price}, citing tighter spreads, native limit-order support, and scalability without per-market subsidy. In a CLOB, $\rho$ is the piecewise-constant function read directly off the order book, and $L_\rho$ captures the depth-induced slippage as the trader walks the book.

\section{Omitted  proofs} \label{a:theory}
 
\subsection{Proof of Theorem \ref{thm:betting-unique} } \label{subsec:proof-main-thm} 
For any strictly proper scoring rule $S$, we show that if a betting strategy $\mbf{s}(\mbf{p},\mbf{q})$ is intrinsically different from $\mbf{s}_G(\mbf{p},\mbf{q})$, then we can find some $\mbf{p} \not = \mbf{q} $ and $ \mbf{p}^*$ such that   $S(\mbf{p};\mbf{p^*}) \geq  S(\mbf{q};\mbf{p^*})$ but $\bm{s}(\mbf{p},\mbf{q})$ has non-positive expected profit under $\mbf{p}^*$. 

Given any two intrinsically different betting strategies $\mbf{s}(\mbf{p},\mbf{q})$ and $\mbf{s}_G(\mbf{p},\mbf{q})$, we first apply constant shifting and rescaling  to both strategies to normalize them so that $\mbf{s}(\mbf{p},\mbf{q}) \cdot \one = \mbf{s}_G(\mbf{p},\mbf{q}) \cdot \one = 0$ (via constant shifting) and  $||\mbf{s}(\mbf{p},\mbf{q})|| = ||\mbf{s}_G(\mbf{p},\mbf{q})|| = 1 $ (via rescaling). No transformations are needed if $\mbf{s}(\mbf{p},\mbf{q}) = \mbf{0}$.  Since $\mbf{s}(\mbf{p},\mbf{q})$ and $\mbf{s}_G(\mbf{p},\mbf{q})$ are intrinsically different, after normalizations there still exists an interior market price $\mbf{q} \in \te{int}(\Delta_{[K]})$, a sequence of  forecasts converging to $\mbf{q}$, i.e., $ \{ \mbf{p}_t \}_{t=1}^\infty \to \mbf{q}$,  yet a small constant $\epsilon$  such that $||\mbf{s}(\mbf{p}_t,\mbf{q}) -  \mbf{s}_G(\mbf{p}_t,\mbf{q})|| \geq \epsilon$  for every $t = 1, 2, \cdots$. The sequences $  \{ \mbf{s}(\mbf{p}_t, \mbf{q}) \}_{t=1}^\infty $ and $\{ \mbf{s}_G(\mbf{p}_t, \mbf{q}) \}_{t=1}^\infty$ are all bounded, so there exists subsequences that converge. Slightly overloading the notation, let $$ \qquad \{ \mbf{s}(\mbf{p}_t, \mbf{q}) \}_{t=1}^\infty  \to \mbf{\hat{s}} \quad \text{and} \quad \{ \mbf{s}_G(\mbf{p}_t, \mbf{q}) \}_{t=1}^\infty  \to \mbf{s}^*.$$ 

We have $|| \mbf{\hat{s}} - \mbf{s}^*||\geq \epsilon$.  
 Next we shall construct a $\mbf{p}^*_t$ for every $t$ such that $S(\mbf{p}_t, \mbf{p}^*_t) \geq  S(\mbf{q}, \mbf{p}^*_t)$ but the expected profit of $\mbf{s}(\mbf{p}_t, \mbf{q})$ is negative for large enough $t$, hence disproving the robust profitability of  $\mbf{s}(\mbf{p}_t, \mbf{q})$.  

We start by handling the corner case with $K = 2$. In this case, after normalization via rescaling and constant shifting, the betting strategy is either $(1/\sqrt{2}, - 1/\sqrt{2})^\top$ or $(-1/\sqrt{2}, 1/\sqrt{2})^\top$ or $\mbf{0}$, which essentially buys  $\sqrt{2}$ of \emph{YES}, or  buys $\sqrt{2}$ of \emph{NO}, or buys nothing. Since $||\mbf{s}(\mbf{p}_t,\mbf{q}) -  \mbf{s}_G(\mbf{p}_t,\mbf{q})|| \geq \epsilon$, they must correspond to these two of these three different strategies respectively. Since $\mbf{p}_t \not = \mbf{q}$, $\mbf{s}_G(\mbf{p}_t,\mbf{q})$ cannot be $\mbf{0}$.  If  $\mbf{s}_G(\mbf{p}_t,\mbf{q})$ buys \emph{YES}, then $\mbf{p}^*_t = (1,0)^\top $ induces higher score for $\mbf{p}_t$ than that of the market $\mbf{q}$ but induces non-positive profit for the other betting strategy $\mbf{s}(\mbf{p}_t,\mbf{q})$, as desired. The case when $\mbf{s}_G(\mbf{p}_t,\mbf{q})$ buys \emph{NO} is symmetric.

In the remainder of the proof,  we consider the non-trivial case $K>2$. We claim that there exists a constant  $\mbf{v}$ such that (1)  $\mbf{v} + \mbf{q} \in \Delta_k$ or equivalently $\mbf{v} \cdot \one = 0$, and (2) $\mbf{v} \cdot \mbf{s}^* > c > 0 > -c > \mbf{v} \cdot \mbf{\hat{s}} $ for some small positive constant $c$. To see this, let $\one^\perp = \{ \mbf{v} \in \mathbb{R}^K: \mbf{v} \cdot \one = 0 \} $ denote the subspace that is perpendicular to the all one vector $\one$. Since
 $\mbf{q} \in \te{int}(\Delta_{[K]}) $, $\Delta_{[K]}$ contains a full-dimensional unit ball within the subspace $\one^\perp$ that is centered around $\mbf{q}$. Due to normalization of betting strategy $\mbf{s}, \mbf{s}_G $, we know $\mbf{\bar{s}}, \mbf{\bar{s}}_G \in  \one^\perp$ and so is their difference. Since $|| \mbf{\hat{s}} - \mbf{s}^*||\geq \epsilon$,  there must exist some $\mbf{v}  \in  \one^\perp$ such that $\mbf{v} \cdot \mbf{s}^* > c > 0 > -c > \mbf{v} \cdot \mbf{\hat{s}} $, as desired. 

We construct $\mbf{p}^*_t  =  \mbf{q} + \mbf{v} \in \Delta_K$ (the same for every $t$), with the $\mbf{v}$ as chosen above. The expected profit of betting strategy  $\mbf{s}(\mbf{p}_t, \mbf{q})$ under this $\mbf{p}^*_t$ satisfies
 \begin{eqnarray*}
 (\mbf{p}^*_t - \mbf{q}) \cdot  \mbf{s}(\mbf{p}_t, \mbf{q}) &=&    \mbf{v}    \cdot  \mbf{s}(\mbf{p}_t, \mbf{q}) \\
 & = &    \mbf{v} \cdot  [\mbf{s}(\mbf{p}_t, \mbf{q}) -  \mbf{\hat{s}}] + \mbf{v} \cdot  \mbf{\hat{s}} \\ 
& < & 0 \quad  \text{ for any $t$ large enough,}  
\end{eqnarray*}  
because  the first term goes to $0$ due to $\mbf{s}(\mbf{p}_t, \mbf{q}) \to   \mbf{\hat{s}} $, whereas the second term  $ \mbf{v} \cdot   \mbf{\hat{s}} \,  (<-c)  $ is a negative constant. 

However, we show that $\mbf{p}_t$'s expected score under ground truth probability $\mbf{p}^*_t$  is positive. Let $G$ be the associated potential function of strictly proper scoring rule $S$, i.e.,  $S(\mbf{p}, y) = G(\mbf{p}) + \nab G(\mbf{p}) \cd (\one_y - \mbf{p}).$ Then we have  
\begin{eqnarray*}
  & &   S(\mbf{p}_t, \mbf{p}^*_t) - S(\mbf{q}, \mbf{p}^*_t) \\
    &=& G(\mbf{p}_t) + \nab G(\mbf{p}_t) \cd (\mbf{p}^*_t - \mbf{p}_t) - G(\mbf{q}) -  \nab G(\mbf{q}) \cd (\mbf{p}^*_t - \mbf{q}) \\
     &=& G(\mbf{p}_t)  - G(\mbf{q}) -  \nab G(\mbf{q}) \cd (\mbf{p}^*_t -  \mbf{q}) + \nab G(\mbf{p}_t) \cd (\mbf{p}^*_t - \mbf{p}_t)  \\ 
       &=& G(\mbf{p}_t)  - G(\mbf{q}) -  \nab G(\mbf{q}) \cd (\mbf{p}_t - \mbf{q}) -  \nab G(\mbf{q}) \cd (\mbf{p}^*_t - \mbf{p}_t) +  \nab G(\mbf{p}_t) \cd (\mbf{p}^*_t - \mbf{p}_t)  \\ 
        &=& D_G(\mbf{p}_t, \mbf{q}) +   [\nab G(\mbf{p}_t) -  \nab G(\mbf{q})]  \cd (\mbf{p}^*_t - \mbf{p}_t)  \\
        & = & D_G(\mbf{p}_t, \mbf{q}) + \mbf{s}_G(\mbf{p}_t, \mbf{q}) \cdot (\mbf{q} + \mbf{v} - \mbf{p}_t )  \\ 
        & = & D_G(\mbf{p}_t, \mbf{q}) + \mbf{s}_G(\mbf{p}_t, \mbf{q}) \cdot (\mbf{q}   - \mbf{p}_t )  +  (\mbf{s}_G(\mbf{p}_t, \mbf{q}) - \mbf{s}^*) \cdot   \mbf{v} +   \mbf{s}^* \cdot   \mbf{v}  \\ 
       & > & 0 \quad  \text{ for any $t$ large enough,}   
\end{eqnarray*}  
because the  Bregman divergence $D_G(\mbf{p}_t, \mbf{q}) \geq  0$ for the convex $G$, the last term $\mbf{s}^* \cdot   \mbf{v}   $ is larger than the positive constant $c$ due to our choice of $\mbf{v} $,  whereas the two middle terms both go to $0$ since $\mbf{p}_t \to \mbf{q}$,  $\mbf{s}_G(\mbf{p}_t, \mbf{q})  \to  \mbf{s}^* $ whereas $\mbf{s}_G(\mbf{p}_t, \mbf{q}),\mbf{v} $ are all bounded vectors.  

This concludes the proof of the second part. That is,  we have found $\mbf{p}_t, \mbf{p}^*_t, \mbf{q}$ for large enough $t$ such that $S(\mbf{p}_t, \mbf{p}^*_t) - S(\mbf{q}, \mbf{p}^*_t) > 0$ but the expected profit of the betting strategy  $\mbf{s}(\mbf{p}_t, \mbf{q})$ under this $\mbf{p}^*_t$, i.e., $
 (\mbf{p}^*_t - \mbf{q}) \cdot  \mbf{s}(\mbf{p}_t, \mbf{q})$, is strictly negative.

\subsection{Proof of \cref{thm:proper} } \label{a:characterization}
The ``only if'' direction is precisely   \cref{thm:betting} since proper betting guarantees robust profitability. We thus only need to prove the ``if'' direction below. The challenge here is to show that if $G$ is not strictly convex, any betting strategy -- not just proper betting -- cannot be robustly profitable. That is, for any not strictly convex $G$ and any betting strategy $\mbf{s} (\mbf{p} ,   \mbf{q}) $, there exists  $\bm{p}^* $ and $\mbf{p} \not =  \mbf{q}$ that satisfy  $S(\mbf{p};\mbf{p^*}) \geq S(\mbf{q};\mbf{p^*})$ \footnote{Scoring rule $S$ is defined as   $S(\mbf{p}, y) = G(\mbf{p}) + \nab G(\mbf{p}) \cd (\one_y - \mbf{p})$.  } but yield non-positive profit under $\mbf{s}$.  

Since $G$ is not strictly convex, there must exist some $\mbf{p}$ and $\mbf{q}$ such that $\mbf{q}\in \text{int}(\Delta_{[K]})$ and
    \begin{equation} \label{eq:bregman-negative-copy}
        D_G(\mbf{q}, \mbf{p}) = G(\mbf{q}) - G(\mbf{p}) - \nab G(\mbf{p}) \cd (\mbf{q} - \mbf{p}) = -\kappa < 0.
    \end{equation}
Given the $\mbf{p}, \mbf{q}$   pair satisfying the above condition,  we now   construct a corresponding $\mbf{p}^*$. Let $\mbf{s}$ be any betting strategy, $\mbf{\hat{s}} =  \frac{ \mbf{s} (\mbf{p} ,   \mbf{q}) }{ || \mbf{s} (\mbf{p} ,   \mbf{q})|| } $ denote the direction of the betting strategy at the above $\mbf{p} ,   \mbf{q}$, and  $L =   || \mbf{s}_G (\mbf{p} ,   \mbf{q})|| $ denote 2-norm of the proper betting strategy.   We construct $\mbf{p}^* = \mbf{q} - \frac{\kappa}{2L} \mbf{\hat{s}} $. 

Next we show that the constructed $\mbf{p}, \mbf{q}, \mbf{p}^*$ have negative profit but satisfy $S(\mbf{p};\mbf{p^*}) \geq S(\mbf{q};\mbf{p^*})$. Notably, this is an interesting situation where the market price $\mbf{q}$ may be much closer to the ground truth $\mbf{p^*}$ than the forecast $\mbf{p}$, but has worse score than $\mbf{p}$. This is intrinsically due to the non-convexity of $G$. 

The difference of the expected score is analyzed as follows:
\begin{eqnarray*}
    S(\mbf{p}, \mbf{p}^* ) -   S(\mbf{q}, \mbf{p}^* )  &=& G(\mbf{p}) + \nab G(\mbf{p}) \cd ( \mbf{p}^* - \mbf{p})  - G(\mbf{q}) - \nab G(\mbf{q}) \cd ( \mbf{p}^* - \mbf{q}) \\ 
    & = & G(\mbf{p})  - G(\mbf{q}) + \nab G(\mbf{p}) \cd ( \mbf{q}  - \mbf{p})  - [\nab G(\mbf{q}) - \nab G(\mbf{p})   ]  \cd ( \mbf{p}^* - \mbf{q}) \\ 
    & = & - D_G(\mbf{q}, \mbf{p}) + \mbf{s}_G (\mbf{p} ,   \mbf{q}) \cd ( \mbf{p}^* - \mbf{q}) \\
    & \geq & \kappa - L \cdot \kappa/(2L) > 0
\end{eqnarray*}

The profit of $\mbf{s} (\mbf{p} ,   \mbf{q})$ is analyzed as follows:
\begin{eqnarray*}
      ( \mbf{p}^* - \mbf{q}) \cd \mbf{s}  (\mbf{p} ,   \mbf{q})   = - \frac{\kappa}{2L} \mbf{\hat{s}} \cdot   \mbf{s}  (\mbf{p} ,   \mbf{q}) = - \frac{\kappa}{2L}  || \mbf{s}  (\mbf{p} ,   \mbf{q}) ||  
     < 0
\end{eqnarray*}

 \subsection{Proof of \cref{prop:amm-compare}}\label{append:proof-amm-compare}
 
We only need to show   $D_G(\mbf{q}, \mbf{p})  =  L_\rho (\mbf{s}^*;\mbf{q}) $  in any AMM governed by a strictly convex cost function $C(\mbf{x})$, whose convex conjugate is   $G(\mbf{p})$. The rest of the statement then follows from   \cref{thm:betting}. 

Let $\mbf{s}^* = \nabla G(\mbf{p}) - \nabla G(\mbf{q})$ denote the proper betting. We compute
\begin{eqnarray*}
    L_\rho (\mbf{s}^*;\mbf{q})  &=&  C(\mbf{s}^* + \mbf{x}) -  C(\mbf{x}) - \nabla C (\mbf{x}) \cdot \mbf{s}^* \\ 
    & = & D_C(\mbf{s}^* + \mbf{x},  \mbf{x} ) \\ 
    & = & D_G ( (\nabla C)^{-1}(\mbf{s}^* + \mbf{x}), (\nabla C)^{-1}(\mbf{x})) \\
    & = & D_G (\mbf{p}, \mbf{q}) 
\end{eqnarray*}
where the third equation is due to the fact that if $C,G$ are convex conjugates, then $D_C( \mbf{x}', \mbf{x}) = D_G( (\nabla C)^{-1}(\mbf{x}'), (\nabla C)^{-1}(\mbf{x}))$.

\section{Extensions of \cref{thm:betting}} \label{a:extensions}

\subsection{Proper betting under nonzero bid-ask spread} \label{a:spread}

In real prediction markets, low liquidity and platform fees create a positive gap between the prices at which a contract can be bought and sold. We extend \cref{thm:betting} to this setting via the following bid-ask convention. For each outcome $k\in [K]$, let $q^+_k$ denote the ask price at which a \emph{YES} contract on outcome $k$ (paying $1$ if $y=k$, otherwise $0$) can be bought, and $q^-_k$ the ask price at which the corresponding \emph{NO} contract (paying $1$ if $y\neq k$, otherwise $0$) can be bought. Absence of riskless arbitrage on outcome $k$ requires $q^+_k + q^-_k \ge 1$, with equality recovering the spread-free setting of \cref{sec:prelim}; in practice the inequality is strict.

\paragraph{Bid-ask-aware proper bet.}
Treat each outcome $k$ as an independent binary event $y_k := \one[y=k] \in \{0,1\}$ and extend the scoring rule to vectors $\mbf{p}\in[0,1]^K$ coordinate-wise:
$$S(\mbf{p}, y) \;:=\; \sum_{k=1}^K S(p_k, y_k),$$
where $S(p_k, y_k)$ is the score for binary outcome of \emph{YES/NO} $k$, in reduced form with 1-D parameter $p_k$. Let $G: [0,1] \to \mathbb{R}$ be its associated convex potential function.  
The Bregman divergence and proper-bet construction extend coordinate-wise as well. Given a model prediction $\mbf{p}\in[0,1]^K$ and prices $(\mbf{q}^+, \mbf{q}^-)$, define the bid-ask-aware proper bet $\mbf{s}\in\bR^K$ by
$$s_k \;:=\; \begin{cases}
\nab G(p_k) - \nab G(q^+_k) & \text{if } p_k > q^+_k \quad \text{(buy YES on $k$ at $q^+_k$)}, \\
\nab G(p_k) - \nab G(1 - q^-_k) & \text{if } p_k < 1 - q^-_k \quad \text{(buy NO on $k$ at $q^-_k$)}, \\
0 & \text{if } 1 - q^-_k \le p_k \le q^+_k.
\end{cases}$$
To see this is a valid betting strategy under non-zero bid/ask spread, we observe that $G(p_k) - \nab G(q^+_k) \geq 0 $ when $p_k > q^+_k$ due to convexity of $G$ and  $G(p_k) - \nab G(1 - q^-_k)  > 0$ for similar reasons. Hence, $s_k$ never needs to  buy  the other side (which may have inconsistent price due to non-zero bid/ask spread).   
The middle case is the spread-induced no-bet zone: the model's edge is too small to overcome the spread on outcome $k$ (the interval is non-empty since $q^+_k + q^-_k \ge 1$). Setting $\tilde{q}_k := q^+_k$, $1 - q^-_k$, or $p_k$ in the three cases respectively, we always have $s_k = \nab G(p_k) - \nab G(\tilde{q}_k)$, and the realized profit is $$\mbf{s}\cd (\one_y - \tilde{\mbf{q}}) = \sum_{k=1}^K s_k(y_k - \tilde{q}_k).$$

\begin{corollary}[Profitability under bid-ask spread (omitting liquidity loss)]\label{cor:spread}
Let $\mbf{s},  \tilde{\mbf{q}}$ be as defined above and $  S(\tilde{\mbf{q}}, y) = \sum_{k=1}^K S(\tilde{q}_k, y_k)$  be the bid-ask-aware market score.  Then the bid-ask-aware proper bet has expected profit admits the following decomposition
\begin{equation}\label{eq:spread-decomp}
    \bE_{y\sim\mbf{p}^*}\bb{\mbf{s}\cd (\one_y - \tilde{\mbf{q}})} \;=\; \ub{\bE_{y\sim\mbf{p}^*}\bb{S(\mbf{p},y) - S(\tilde{\mbf{q}}, y) }}_{\text{score gap}} \;+\; \ub{\sum_{k=1}^K D_G(\tilde{q}_k, p_k)}_{\text{Bregman bonus}}   
\end{equation}
\end{corollary}

\begin{proof}
For each coordinate $k$, $s_k = \nab G(p_k) - \nab G(\tilde{q}_k)$ is the binary proper bet with respect to the reference $\tilde{q}_k$. Apply the binary version of \cref{lem:decomposition}:
$$s_k(y_k - \tilde{q}_k) \;=\; \bp{\nab G(p_k) - \nab G(\tilde{q}_k)}(y_k - \tilde{q}_k) \;=\; \bb{S(p_k, y_k) - S(\tilde{q}_k, y_k)} + D_G(\tilde{q}_k, p_k).$$
Summing over $k\in [K]$ and taking expectation under $y\sim \mbf{p}^*$, minus the liquidity loss, yields \cref{eq:spread-decomp}. 
\end{proof}

\begin{remark}[Correlated outcomes.]
 \cref{cor:spread} holds even for outcomes $\{ 1, \cdots, K \}$ that need not to be disjoint but can be arbitrarily correlated. In this case, we simply treat each outcome $k$ as a separate (though possibly correlated) binary event. 
\end{remark}  

\paragraph{Cost of spread.}
The reference $\tilde{\mbf{q}}$ depends on which side of the spread each coordinate lands on, so an accuracy edge against $\tilde{\mbf{q}}$ is harder to attain than one against any spread-free price $\bar{\mbf{q}}$ with $\bar{q}_k \in [1-q^-_k, q^+_k]$. Specifically, the score gap against $\tilde{\mbf{q}}$ decomposes as
$$S(\mbf{p}, y) - S(\tilde{\mbf{q}}, y) \;=\; \bb{S(\mbf{p}, y) - S(\bar{\mbf{q}}, y)} \;+\; \ub{\bb{S(\bar{\mbf{q}}, y) - S(\tilde{\mbf{q}}, y)}}_{\text{cost of spread} \;\le\; 0},$$
where the spread cost vanishes when $q^+_k + q^-_k = 1$ (zero spread). The bid-ask-aware result thus interpolates between the spread-free guarantee of \cref{thm:betting} and a strictly weaker condition that absorbs the per-trade spread cost.



\subsection{Proper betting for long time horizon}
\label{sec:portfolio}

For events that take days or weeks to resolve, both forecaster and market predictions evolve over time. A forecaster can make different prediction $\mbf{p}^t$ at different time $t = 1, \cdots, T$ and makes trades based on $\mbf{p}^t$ as well as then market price $\mbf{q}^t$.   We extend \cref{thm:betting} to two natural multi-period strategies, each compounding the per-round proper bet but differing in what they hold and what accuracy edge they require. In both corollaries below, we omit liquidity loss for simplicity.  

\begin{corollary}[Fundamental-driven strategy]
\label{cor:fundamental}
The \emph{fundamental} strategy executes the proper bet $\mbf{s}^t := \nab G(\mbf{p}^t) - \nab G(\mbf{q}^t)$ at each time $t$ and holds all positions to resolution. 
Then the cumulative expected profit at resolution decomposes as
\begin{equation}
\label{eq:fundamental}
    \sum_{t=1}^T \mbf{s}^t \cd (\mbf{p}^* - \mbf{q}^t) \;=\; \sum_{t=1}^T \bb{S(\mbf{p}^t;\mbf{p}^*) - S(\mbf{q}^t;\mbf{p}^*)} \;+\; \sum_{t=1}^T D_G(\mbf{q}^t, \mbf{p}^t)  
\end{equation}
\end{corollary}

\begin{corollary}[Momentum-driven strategy]
\label{cor:momentum}
The \emph{momentum} strategy maintains the proper position $\mbf{x}^t := \nab G(\mbf{p}^t) - \nab G(\mbf{q}^t)$ at each time $t$, rebalancing after every market update; the marginal trade at time $t$ is $\mbf{s}^t := \mbf{x}^t - \mbf{x}^{t-1}$ with $\mbf{x}^0 := \mbf{0}$. 
Then the cumulative mark-to-market return decomposes as
\begin{equation}
\label{eq:momentum}
    \sum_{t=1}^T \mbf{x}^t \cd (\mbf{q}^{t+1} - \mbf{q}^t) \;=\; \sum_{t=1}^T \bb{S(\mbf{p}^t;\mbf{q}^{t+1}) - S(\mbf{q}^t;\mbf{q}^{t+1})} \;+\; \sum_{t=1}^T D_G(\mbf{q}^t, \mbf{p}^t)    
\end{equation}
Moreover, the accumulated profit of executing trades $\mbf{s}^t$ at prices $\mbf{q}^t$ and settling at $\mbf{q}^{T+1}$ can be equivalently expressed in the following trade form $$\sum_{t=1}^T \mbf{x}^t \cd (\mbf{q}^{t+1} - \mbf{q}^t) = \sum_{t=1}^T \mbf{s}^t \cd (\mbf{q}^{T+1} - \mbf{q}^t),$$   
\end{corollary}

The two corollaries differ in what plays the role of the ``ground truth'' in the proper-bet decomposition: \cref{cor:fundamental} uses the unobservable $\mbf{p}^*$ that resolves the event, so the trader needs an accuracy edge against the truth and is exposed across the entire horizon; \cref{cor:momentum} uses the next-period market price $\mbf{q}^{t+1}$, so the trader only needs to predict the market's next move and is exposed only one period at a time.

\begin{proof}
For \cref{cor:fundamental}, apply \cref{lem:decomposition} at each round with realized outcome $y$:
$$\mbf{s}^t \cd (\one_y - \mbf{q}^t) \;=\; \bb{S(\mbf{p}^t, y) - S(\mbf{q}^t, y)} + D_G(\mbf{q}^t, \mbf{p}^t).$$
Summing over $t$ and taking expectation under $y \sim \mbf{p}^*$ gives \cref{eq:fundamental}. The score-gap sum is strictly positive by hypothesis and the Bregman sum is non-negative, so the total is positive.

For \cref{cor:momentum}, apply \cref{lem:decomposition} (in expectation form) at each round with substitutions $\mbf{p}^* \to \mbf{q}^{t+1}$, $\mbf{p} \to \mbf{p}^t$, $\mbf{q} \to \mbf{q}^t$:
$$\mbf{x}^t \cd (\mbf{q}^{t+1} - \mbf{q}^t) \;=\; \bb{S(\mbf{p}^t;\mbf{q}^{t+1}) - S(\mbf{q}^t;\mbf{q}^{t+1})} + D_G(\mbf{q}^t, \mbf{p}^t).$$
Summing gives \cref{eq:momentum}. The trade-form equivalence follows from substituting $\mbf{x}^t = \sum_{r=1}^t \mbf{s}^r$ and exchanging the order of summation:
$$\sum_{t=1}^T \mbf{x}^t \cd (\mbf{q}^{t+1} - \mbf{q}^t) \;=\; \sum_{r=1}^T \mbf{s}^r \cd \sum_{t=r}^T (\mbf{q}^{t+1} - \mbf{q}^t) \;=\; \sum_{r=1}^T \mbf{s}^r \cd (\mbf{q}^{T+1} - \mbf{q}^r),$$
where the inner sum telescopes by $$\sum_{t=r}^T (\mbf{q}^{t+1} - \mbf{q}^t) = \mbf{q}^{T+1} - \mbf{q}^r.$$ Strict positivity follows similarly to the fundamental-driven strategy.
\end{proof}
\section{Additional experiments}

\subsection{Model details}
\label{app:modelDetails}

\begin{table}[t]
\centering
\small
\setlength{\tabcolsep}{4pt}
\renewcommand{\arraystretch}{1.05}
\caption{Calibration and profitability metrics by model and persona group. $N$ is the number of market-level predictions per model. ECE is expected calibration error. Quad. Score represents the model's score under the quadratic scoring rule. $\pm$ values are the bootstrapped standard errors over $1{,}000$ resamples. Shaded rows report persona-group averages.}
\label{tab:calibration-summary}

\begin{tabular}{l r r r r}
\toprule
\textbf{Model} & \textbf{$N$ markets} & \textbf{ECE} & \textbf{Quad. Score} & \textbf{ROI} \\
\midrule

GPT-5.2 (High)    & 11{,}789 & 0.021 \tiny{$\pm 0.005$} & 0.916 \tiny{$\pm 0.006$} & $-7.8\%$ \tiny{$\pm 3.6$} (Log) \\
GPT-5.2 (Base)    & 11{,}789 & 0.013 \tiny{$\pm 0.003$} & 0.916 \tiny{$\pm 0.005$} & $-4.4\%$ \tiny{$\pm 3.9$} (Log) \\
Grok 4.1 Fast     & 11{,}407 & 0.027 \tiny{$\pm 0.004$} & 0.913 \tiny{$\pm 0.005$} & $-9.1\%$ \tiny{$\pm 4.5$} (Log) \\
Claude Sonnet 4.5 & 11{,}785 & 0.023 \tiny{$\pm 0.003$} & 0.911 \tiny{$\pm 0.005$} & $-8.1\%$ \tiny{$\pm 3.5$} (Log) \\
Kimi K2 Thinking  & 11{,}811 & 0.024 \tiny{$\pm 0.004$} & 0.912 \tiny{$\pm 0.005$} & $-8.7\%$ \tiny{$\pm 3.8$} (Log) \\
DeepSeek V3.2     & 11{,}071 & 0.031 \tiny{$\pm 0.005$} & 0.899 \tiny{$\pm 0.006$} & $-12.8\%$ \tiny{$\pm 3.4$} (Log) \\
Minimax M2        & 11{,}688 & 0.028 \tiny{$\pm 0.005$} & 0.900 \tiny{$\pm 0.006$} & $-9.4\%$ \tiny{$\pm 3.8$} (Log) \\

\rowcolor{orange!7}
\textbf{\textit{Brittle (avg)}} 
& \textbf{11{,}620} 
& \textbf{0.024} 
& \textbf{0.910} 
& $\mathbf{-8.6\%}$ \textbf{(Log)} \\

\midrule

LLaMA 4 Maverick  & 29{,}613 & 0.069 \tiny{$\pm 0.004$} & 0.871 \tiny{$\pm 0.005$} & $-4.9\%$ \tiny{$\pm 2.2$} (Log) \\
Qwen 3 235B       & 29{,}923 & 0.078 \tiny{$\pm 0.004$} & 0.868 \tiny{$\pm 0.004$} & $-11.2\%$ \tiny{$\pm 1.9$} (Log) \\
DeepSeek R1       & 29{,}759 & 0.100 \tiny{$\pm 0.005$} & 0.848 \tiny{$\pm 0.005$} & $-5.4\%$ \tiny{$\pm 1.7$} (Log) \\

\rowcolor{teal!7}
\textbf{\textit{Dispersed (avg)}} 
& \textbf{29{,}765} 
& \textbf{0.082} 
& \textbf{0.862} 
& $\mathbf{-7.2\%}$ \textbf{(Log)} \\

\midrule

Gemini 3        & 9{,}721 & 0.017 \tiny{$\pm 0.003$} & 0.936 \tiny{$\pm 0.004$} & $+14.2\%$ \tiny{$\pm 8.8$} (Brier) \\
Claude Opus 4.6 & 3{,}172 & 0.011 \tiny{$\pm 0.004$} & 0.949 \tiny{$\pm 0.007$} & $+17.5\%$ \tiny{$\pm 7.2$} (Brier) \\

\rowcolor{purple!7}
\textbf{\textit{Conservative (avg)}}
& \textbf{6{,}446}
& \textbf{0.014}
& \textbf{0.942}
& $\mathbf{+15.8\%}$ \textbf{(Brier)} \\

\bottomrule
\end{tabular}
\end{table}

\cref{tab:calibration-summary} reports the LLMs evaluated across three metrics -- expected calibration error (ECE), the score under the quadratic scoring rule, and ROI under each model's best proper betting strategy. GPT-5.2 (Base/High) correspond to different reasoning levels of the same underlying model. The number of markets varies across forecasters because models were released at different times, resulting in differing amounts of collected data through the Prophet Arena pipeline; to ensure comparability, we standardize evaluations in \cref{sec:experiments}. 

Interestingly, while the Brittle group has lower ECE and the higher aggregate Brier accuracy as compared to the Dispersed group, it is also more unprofitable. As such, aggregate accuracy and calibration can therefore coexist with systematic capital loss, due to the importance of the Bregman divergence term.

\subsection{Implementation details of betting strategies}
\label{app:appCanonical}

\paragraph{Setup.} 
We consider a collection of $n$ markets with forecast–price pairs $\{p_i, q_i\}_{i\in[n]}$. 
A betting strategy is defined by a set of budget allocation weights $\{w_i\}_{i\in[n]}$, which encode both direction and size: if $w_i > 0$, we buy YES on market $i$ with allocation $w_i$; if $w_i < 0$, we buy NO (equivalently, sell YES) with allocation $-w_i$. 
Without loss of generality, we restrict $w_i \in [-1,1]$ for all $i \in [n]$ and normalize total exposure such that $\sum_{i\in[n]} |w_i| = 1$. 




\paragraph{Proper Betting Strategies.} Proper betting strategies can be constructed as follows:
\begin{enumerate}[leftmargin=*]
    \item \textbf{Brier:} Allocates linearly in the margin:
    $w_i \propto p_i - q_i .$

    \item \textbf{Logarithmic:} Scales allocations according to:
    $
    w_i \propto \log p_i - \log q_i .
    $

    \item \textbf{Spherical:} Allocates based on the difference between the $L_2$-normalized forecast and market price vectors:
    $
    w_i \propto \frac{p_i}{\norm{p}} - \frac{q_i}{\norm{q}}.
    $
\end{enumerate}

\paragraph{Baseline Betting Strategies.} We also construct a few betting strategies from well-motivated heuristics as baselines: 

\begin{enumerate}[leftmargin=*]
\item \textbf{Max-Margin:}
Strategies that bet the margin.
\begin{itemize}[leftmargin=.5em]
\item \emph{Market-Level:} 
$w_i \propto \sgn(p_i - q_i), \forall i \in [n]$. 
This strategy places a unit bet on every market, with direction determined by whether $p_i > q_i$.
\item \emph{Grouped:} 
Let $\{\mathcal{I}_k\}_{k=1}^m$ denote a partition of markets into groups corresponding to the same underlying question. For each group $k$, define
\[
i^\star(k) = \arg\max_{i \in \mathcal{I}_k} \ab{p_i - q_i}.
\]
This strategy allocates the budget uniformly across groups, placing a single bet per group on the market with the largest margin:
\[
w_{i^\star(k)} = \frac{\sgn(p_{i^\star(k)} - q_{i^\star(k)})}{m}, \quad w_i = 0,\quad \forall i \notin \{i^\star(k)\}_{k=1}^m.
\]

\end{itemize}

\item \textbf{Inverse-Margin: } 
$w_i \propto \frac{\sgn(p_i - q_i)}{|p_i - q_i|}, \forall i \in [n]$. 
This heuristic places more weight on markets with smaller margins, under the hypothesis that forecasts are more reliable when deviations from the market are modest.

\item \textbf{Kelly-Alike:} $w_i \propto \frac{p_i - q_i}{1 - q_i}, \forall i\in [n]$. This heuristic strategy mimics the ratio in the Kelly criterion.

\item \textbf{Kelly Criterion \citep{kelly1956new}:}
Unlike the previous strategies, which allocate a fixed budget $B$ across markets via weights $w_i$ with the goal of maximizing expected profit, Kelly sizing determines the fraction of wealth to wager on each individual market sequentially to maximize the expected log-growth rate. For a binary market with forecast $p_i$ and market price $q_i$, a unit stake on YES pays $1/q_i$, giving net odds $b_i = (1-q_i)/q_i$. Maximizing
\[
\mathbb{E}[\log W_{i+1}]
= p_i \log(1 + f_i b_i) + (1 - p_i)\log(1 - f_i)
\]
over the wealth fraction $f_i$ yields the closed form \( f_i = \frac{\lvert p_i - q_i \rvert}{1 - q_i} \). If $f_i$ exceeds 1, we allow for the use of leverage to purchase additional shares.

\end{enumerate}

\cref{tab:app-betting-strategy-std} reports the full comparison of evaluated models across all baseline betting strategies and the proposed Proper (Brier) strategy. The Kelly criterion exhibits particularly high variance in ROI because it compounds the bankroll while effectively treating the forecaster's estimated disagreement with the market as the true edge. When this edge is poorly estimated, Kelly can substantially oversize individual positions. A single sufficiently large loss can therefore deplete the bankroll, causing compounded returns to collapse to $-100\%$ for many models.

\begin{table*}[t]
\centering
\scriptsize
\setlength{\tabcolsep}{3pt}
\renewcommand{\arraystretch}{1.05}

\caption{Pairwise comparison of the Proper Betting (Brier) strategy against each of the five baselines on the expanded dataset, with each column reporting the difference between Proper Betting and the corresponding baseline. Confidence intervals are paired, event-clustered 95\% bootstrap intervals; bold indicates differences where the CI excludes zero.}
\label{tab:app-betting-strategy-std}

\newcommand{\ci}[2]{\text{\scriptsize [#1, #2]}}
\resizebox{\textwidth}{!}{%
\begin{tabular}{@{}lcccccc@{}}
\toprule
Model & $\Delta S$ & $-$ Max (Mkt) & $-$ Max (Grp) & $-$ Inv-Margin & $-$ Kelly-Alike & $-$ Kelly \\
\midrule

Gemini 3.1 Pro (B)
& \textbf{+0.0076} \ci{+0.0014}{+0.0136}
& \textbf{+30.3} \ci{+20.7}{+39.9}
& \textbf{+30.1} \ci{+15.5}{+42.9}
& \textbf{+39.2} \ci{+27.7}{+50.6}
& \textbf{+24.3} \ci{+16.1}{+32.8}
& \textbf{+138.9} \ci{+127.3}{+150.4} \\

Claude Fable 5
& \textbf{+0.0039} \ci{+0.0008}{+0.0073}
& \textbf{+26.8} \ci{+15.3}{+40.0}
& \textbf{+48.3} \ci{+34.2}{+63.1}
& \textbf{+33.6} \ci{+20.0}{+48.9}
& \textbf{+16.8} \ci{+7.3}{+27.4}
& +129.0 \ci{$-1.3\times10^{9}$}{+141.8} \\

Gemini 3.1 Pro (A)
& \textbf{+0.0039} \ci{+0.0016}{+0.0063}
& \textbf{+25.9} \ci{+20.4}{+31.5}
& \textbf{+36.6} \ci{+28.3}{+44.5}
& \textbf{+29.9} \ci{+23.5}{+36.5}
& \textbf{+16.3} \ci{+11.7}{+21.0}
& \textbf{+127.9} \ci{+121.5}{+134.6} \\

GPT-5.6 Sol
& \textbf{+0.0033} \ci{+0.0003}{+0.0067}
& \textbf{+22.4} \ci{+12.2}{+35.4}
& \textbf{+39.9} \ci{+24.3}{+55.7}
& \textbf{+28.5} \ci{+15.7}{+43.7}
& \textbf{+15.7} \ci{+6.0}{+27.1}
& -1220.4 \ci{$-2.2\times10^{8}$}{+122.5} \\

Claude Sonnet 5
& \textbf{+0.0028} \ci{+0.0002}{+0.0057}
& \textbf{+20.1} \ci{+10.3}{+32.1}
& \textbf{+58.8} \ci{+45.0}{+73.2}
& \textbf{+25.7} \ci{+13.4}{+40.1}
& \textbf{+8.3} \ci{+1.3}{+16.9}
& +125.9 \ci{-5766.9}{+133.1} \\

GPT-5.5 (B)
& -0.0059 \ci{-0.0129}{+0.0009}
& \textbf{+13.0} \ci{-1.2}{+29.3}
& \textbf{+20.3} \ci{-0.7}{+41.1}
& \textbf{+18.6} \ci{+0.1}{+38.9}
& +8.6 \ci{-2.7}{+21.4}
& \textbf{+116.8} \ci{+99.8}{+136.0} \\

GPT-5.5 (C)
& -0.0056 \ci{-0.0101}{-0.0012}
& +8.8 \ci{-1.5}{+18.7}
& +3.6 \ci{-28.4}{+27.3}
& +7.0 \ci{-5.5}{+19.4}
& +2.0 \ci{-6.3}{+10.0}
& \textbf{+107.3} \ci{+94.9}{+119.2} \\

Gemini 3.5 Flash
& -0.0188 \ci{-0.0299}{-0.0081}
& -2.4 \ci{-11.9}{+6.8}
& +1.8 \ci{-19.5}{+19.3}
& +1.3 \ci{-10.8}{+13.0}
& +0.1 \ci{-7.5}{+7.6}
& \textbf{+105.9} \ci{+94.2}{+117.7} \\

Qwen 3.5+ (B)
& -0.0342 \ci{-0.0415}{-0.0272}
& \textbf{+8.3} \ci{+0.6}{+16.0}
& +9.6 \ci{-7.7}{+24.2}
& +7.7 \ci{-1.9}{+17.4}
& -0.1 \ci{-5.2}{+5.2}
& \textbf{+105.6} \ci{+96.4}{+114.6} \\

Kimi K3
& -0.0030 \ci{-0.0065}{+0.0004}
& +1.9 \ci{-10.7}{+15.5}
& \textbf{+40.3} \ci{+24.2}{+56.4}
& +4.5 \ci{-12.6}{+22.3}
& -1.9 \ci{-10.9}{+7.9}
& \textbf{+104.8} \ci{+88.8}{+121.6} \\

GPT-5.4
& -0.0203 \ci{-0.0258}{-0.0159}
& +2.3 \ci{-2.9}{+7.6}
& \textbf{+12.5} \ci{+1.9}{+22.5}
& +5.1 \ci{-2.5}{+13.1}
& -2.5 \ci{-5.6}{+0.7}
& \textbf{+99.8} \ci{+92.5}{+107.4} \\

GPT-5.5 (D)
& -0.0065 \ci{-0.0081}{-0.0050}
& -0.3 \ci{-3.8}{+3.1}
& \textbf{+15.7} \ci{+9.4}{+21.7}
& +1.3 \ci{-3.5}{+6.3}
& -3.1 \ci{-6.2}{-0.0}
& \textbf{+99.8} \ci{+95.2}{+104.2} \\

Muse Spark 1.1
& -0.0133 \ci{-0.0203}{-0.0066}
& -4.2 \ci{-18.0}{+11.4}
& +16.5 \ci{-3.2}{+35.5}
& -1.9 \ci{-19.2}{+16.1}
& -2.3 \ci{-11.1}{+7.9}
& \textbf{+98.9} \ci{+82.6}{+116.8} \\

GLM 5.2
& -0.0194 \ci{-0.0248}{-0.0139}
& -2.3 \ci{-10.0}{+5.6}
& \textbf{+22.7} \ci{+9.2}{+34.9}
& -4.3 \ci{-14.6}{+6.3}
& -3.0 \ci{-7.1}{+1.4}
& \textbf{+98.6} \ci{+88.8}{+108.4} \\

Claude Sonnet 4.6 (A)
& -0.0292 \ci{-0.0353}{-0.0233}
& -1.2 \ci{-7.9}{+5.5}
& +7.6 \ci{-5.8}{+19.7}
& -1.1 \ci{-10.9}{+8.6}
& -2.8 \ci{-6.5}{+1.2}
& \textbf{+94.9} \ci{+86.9}{+103.0} \\

Grok 4.5
& -0.0115 \ci{-0.0164}{-0.0067}
& -4.6 \ci{-15.5}{+8.3}
& \textbf{+24.0} \ci{+7.8}{+40.0}
& -4.4 \ci{-18.6}{+11.6}
& -5.1 \ci{-11.9}{+2.8}
& \textbf{+94.5} \ci{+80.9}{+110.0} \\

Claude Sonnet 4.6 (B)
& -0.0236 \ci{-0.0346}{-0.0128}
& -0.9 \ci{-13.8}{+13.2}
& -16.0 \ci{-63.5}{+19.5}
& -10.5 \ci{-37.6}{+12.6}
& -2.4 \ci{-9.2}{+6.1}
& \textbf{+94.2} \ci{+78.5}{+110.9} \\

Gemma 4 31B
& -0.0339 \ci{-0.0410}{-0.0268}
& -3.3 \ci{-9.3}{+2.8}
& +10.1 \ci{-0.6}{+19.9}
& -5.8 \ci{-13.7}{+2.4}
& \textbf{-4.5} \ci{-8.5}{-0.4}
& \textbf{+93.2} \ci{+85.4}{+100.8} \\

\bottomrule
\end{tabular}
}
\end{table*}

\subsection{Empirical decomposition}
\label{app:addModelExps}

\cref{tab:scoring-decomposition-appendix} reports the full decomposition across all models and all evaluated proper betting strategies. Divergence varies substantially across both models and rules, and often offsets large negative score gaps, meaning models less accurate can still yield positive ROI. Rules such as Log tend to amplify divergence, leading to higher upside but also greater dispersion in outcomes, while Brier and Spherical are comparatively more stable. This reinforces that the mapping from forecasts to returns is highly rule-dependent, and that profitability hinges as much on how aggressively a strategy exploits disagreement as on the underlying predictive accuracy.

\begin{table*}[t]
\centering
\small
\setlength{\tabcolsep}{3pt}
\renewcommand{\arraystretch}{1.15}
\caption{\textbf{ROI decomposition on the expanded dataset}, where
$\mathrm{ROI}=\Delta S + D$ for each model and proper scoring rule.
Bold denotes significantly positive score gaps (95\% CI above zero),
and bold italic denotes significantly positive ROI.}
\label{tab:scoring-decomposition-appendix}
\resizebox{\textwidth}{!}{%
\begin{tabular}{lrrrrrrrrr}
\toprule
Model
& Brier $\Delta S$
& Brier $D$
& Brier ROI
& Log $\Delta S$
& Log $D$
& Log ROI
& Sph $\Delta S$
& Sph $D$
& Sph ROI \\
\midrule

Gemini 3.1 Pro (B)
& \best{+11.7}{[+2.4, +21.1]}
& \est{+27.2}{[+22.7, +32.1]}
& \broi{+38.9}{[+27.4, +50.9]}
& \best{+7.5}{[+0.8, +14.0]}
& \est{+20.9}{[+18.2, +23.8]}
& \broi{+28.4}{[+20.4, +36.2]}
& \best{+11.3}{[+3.5, +19.5]}
& \est{+21.9}{[+18.3, +25.8]}
& \broi{+33.2}{[+23.9, +42.9]} \\

Claude Fable 5
& \best{+12.3}{[+3.2, +22.0]}
& \est{+16.7}{[+11.8, +22.2]}
& \broi{+29.0}{[+16.7, +42.7]}
& \best{+8.6}{[+1.3, +15.6]}
& \est{+12.9}{[+9.8, +16.5]}
& \broi{+21.5}{[+12.2, +30.6]}
& \best{+10.2}{[+1.4, +18.6]}
& \est{+11.9}{[+8.5, +15.9]}
& \broi{+22.1}{[+11.6, +33.0]} \\

Gemini 3.1 Pro (A)
& \best{+8.7}{[+3.3, +13.7]}
& \est{+19.2}{[+17.0, +21.7]}
& \broi{+27.9}{[+21.1, +34.3]}
& \best{+4.8}{[+0.8, +8.5]}
& \est{+16.4}{[+14.7, +18.2]}
& \broi{+21.2}{[+16.4, +25.8]}
& \best{+6.5}{[+2.1, +10.8]}
& \est{+13.6}{[+12.1, +15.4]}
& \broi{+20.1}{[+14.9, +25.7]} \\

GPT-5.6 Sol
& \best{+10.9}{[+1.0, +21.7]}
& \est{+15.9}{[+12.1, +21.9]}
& \broi{+26.7}{[+14.4, +41.4]}
& \best{+10.7}{[+3.0, +18.5]}
& \est{+12.6}{[+10.1, +16.1]}
& \broi{+23.3}{[+14.5, +32.4]}
& \best{+12.1}{[+3.2, +21.4]}
& \est{+11.4}{[+8.5, +15.7]}
& \broi{+23.5}{[+13.0, +35.8]} \\

Claude Sonnet 5
& \best{+10.2}{[+0.2, +19.9]}
& \est{+15.8}{[+11.8, +20.9]}
& \broi{+25.9}{[+13.4, +38.9]}
& \best{+9.1}{[+1.9, +16.4]}
& \est{+12.4}{[+9.9, +15.6]}
& \broi{+21.5}{[+13.0, +30.7]}
& \best{+12.3}{[+2.9, +21.7]}
& \est{+12.4}{[+9.5, +16.4]}
& \broi{+24.7}{[+13.9, +36.4]} \\

GPT-5.5 (B)
& \est{-15.4}{[-29.6, -0.4]}
& \est{+32.2}{[+25.8, +39.5]}
& \broi{+16.8}{[+2.3, +33.3]}
& \est{-9.5}{[-20.7, +0.5]}
& \est{+21.2}{[+17.9, +24.9]}
& \broi{+11.7}{[+0.3, +22.2]}
& \est{-11.5}{[-23.0, -0.7]}
& \est{+24.4}{[+19.4, +30.1]}
& \broi{+12.9}{[+0.9, +25.0]} \\

GPT-5.5 (C)
& \est{-15.6}{[-26.3, -5.6]}
& \est{+22.9}{[+18.6, +27.3]}
& \est{+7.3}{[-3.3, +17.3]}
& \est{-13.1}{[-21.4, -5.5]}
& \est{+18.4}{[+15.2, +21.6]}
& \est{+5.2}{[-3.8, +13.9]}
& \est{-14.6}{[-23.4, -6.2]}
& \est{+16.5}{[+13.4, +20.1]}
& \est{+1.9}{[-6.6, +10.7]} \\

Gemini 3.5 Flash
& \est{-26.3}{[-39.9, -14.1]}
& \est{+32.2}{[+26.1, +39.2]}
& \est{+5.9}{[-4.4, +15.8]}
& \est{-17.4}{[-26.8, -8.2]}
& \est{+25.5}{[+21.8, +29.8]}
& \est{+8.1}{[-0.5, +16.4]}
& \est{-22.5}{[-35.4, -10.3]}
& \est{+27.4}{[+21.8, +33.5]}
& \est{+4.9}{[-5.0, +14.8]} \\

Qwen 3.5+ (B)
& \est{-58.4}{[-68.9, -48.4]}
& \est{+64.0}{[+57.7, +70.8]}
& \est{+5.6}{[-2.1, +13.3]}
& \est{-33.1}{[-39.6, -27.0]}
& \est{+38.9}{[+35.6, +41.8]}
& \est{+5.8}{[-0.6, +11.7]}
& \est{-48.6}{[-57.5, -40.5]}
& \est{+48.9}{[+43.5, +54.4]}
& \est{+0.3}{[-6.4, +6.5]} \\

Kimi K3
& \est{-15.0}{[-29.1, -0.7]}
& \est{+19.9}{[+15.2, +25.1]}
& \est{+4.8}{[-8.2, +18.4]}
& \est{-7.4}{[-18.9, +3.7]}
& \est{+13.7}{[+11.7, +16.0]}
& \est{+6.4}{[-4.6, +17.4]}
& \est{-10.4}{[-22.4, +1.5]}
& \est{+14.2}{[+10.4, +18.7]}
& \est{+3.8}{[-7.3, +15.3]} \\

GPT-5.4
& \est{-65.0}{[-82.5, -51.3]}
& \est{+64.8}{[+55.1, +79.3]}
& \est{-0.2}{[-7.8, +7.3]}
& \est{-32.0}{[-39.2, -24.8]}
& \est{+33.2}{[+30.9, +35.7]}
& \est{+1.2}{[-5.4, +8.0]}
& \est{-50.8}{[-63.9, -39.1]}
& \est{+50.4}{[+42.3, +60.6]}
& \est{-0.4}{[-7.0, +6.8]} \\

GPT-5.5 (D)
& \est{-23.8}{[-28.9, -18.4]}
& \est{+23.6}{[+21.4, +26.1]}
& \est{-0.2}{[-4.6, +4.4]}
& \est{-15.1}{[-19.4, -10.7]}
& \est{+16.6}{[+15.6, +17.7]}
& \est{+1.5}{[-2.5, +5.7]}
& \est{-18.1}{[-22.9, -13.4]}
& \est{+17.4}{[+15.6, +19.5]}
& \est{-0.7}{[-4.8, +3.3]} \\

Muse Spark 1.1
& \est{-35.3}{[-50.4, -19.5]}
& \est{+34.2}{[+27.7, +41.1]}
& \est{-1.1}{[-15.2, +13.9]}
& \est{-23.5}{[-35.1, -11.4]}
& \est{+24.4}{[+20.7, +28.1]}
& \est{+0.9}{[-10.2, +12.3]}
& \est{-31.1}{[-44.6, -17.4]}
& \est{+28.5}{[+23.2, +34.6]}
& \est{-2.6}{[-14.7, +10.2]} \\

GLM 5.2
& \est{-50.3}{[-61.5, -39.1]}
& \est{+48.8}{[+42.4, +56.0]}
& \est{-1.4}{[-9.2, +6.9]}
& \est{-27.2}{[-34.5, -19.8]}
& \est{+27.9}{[+25.3, +30.4]}
& \est{+0.7}{[-6.1, +7.4]}
& \est{-38.8}{[-47.7, -29.9]}
& \est{+37.2}{[+32.1, +42.6]}
& \est{-1.6}{[-8.7, +5.7]} \\

Claude Sonnet 4.6 (A)
& \est{-50.4}{[-58.5, -42.5]}
& \est{+45.3}{[+41.0, +49.9]}
& \est{-5.1}{[-11.5, +1.4]}
& \est{-35.9}{[-41.6, -29.8]}
& \est{+31.6}{[+29.4, +33.7]}
& \est{-4.3}{[-9.7, +1.4]}
& \est{-42.1}{[-49.5, -34.9]}
& \est{+35.8}{[+32.2, +39.6]}
& \est{-6.3}{[-12.0, +0.0]} \\

Grok 4.5
& \est{-37.5}{[-50.3, -24.8]}
& \est{+32.0}{[+25.1, +39.4]}
& \est{-5.5}{[-17.4, +7.4]}
& \est{-25.0}{[-33.6, -16.0]}
& \est{+20.4}{[+17.4, +23.6]}
& \est{-4.6}{[-13.6, +4.8]}
& \est{-31.1}{[-42.3, -19.7]}
& \est{+25.5}{[+20.3, +31.2]}
& \est{-5.6}{[-15.8, +5.4]} \\

Claude Sonnet 4.6 (B)
& \est{-38.7}{[-52.8, -24.4]}
& \est{+32.8}{[+26.8, +39.6]}
& \est{-5.8}{[-19.0, +8.5]}
& \est{-34.3}{[-45.6, -22.7]}
& \est{+28.0}{[+24.6, +31.3]}
& \est{-6.3}{[-16.8, +4.2]}
& \est{-35.1}{[-47.4, -23.1]}
& \est{+25.8}{[+21.1, +31.2]}
& \est{-9.3}{[-20.9, +2.3]} \\

Gemma 4 31B
& \est{-52.4}{[-60.9, -43.4]}
& \est{+45.6}{[+40.7, +50.7]}
& \est{-6.8}{[-13.2, -0.1]}
& \est{-33.6}{[-40.3, -27.0]}
& \est{+29.8}{[+27.4, +32.2]}
& \est{-3.8}{[-9.7, +1.8]}
& \est{-41.6}{[-49.0, -34.3]}
& \est{+35.8}{[+31.8, +40.1]}
& \est{-5.9}{[-11.5, -0.2]} \\

\bottomrule
\end{tabular}%
}
\end{table*}

\subsection{Synthetic persona generation}
\label{app:syntheticMethod}

For each persona, we generate a synthetic forecast by perturbing the market price by a persona-specific margin:
\[
p_i = q_i \pm m_i \quad \text{(clipped to $[\varepsilon,\, 1 - \varepsilon]$),}
\]
where $m_i \geq 0$ is the margin magnitude drawn from a persona-specific distribution, and the sign is chosen toward the realized outcome with probability $a_i$. A single global offset on $a_i$ is calibrated per persona so that all personas attain the same Brier score within each regime.

\paragraph{Synthetic persona construction.}
Conservative uses $m \sim \mathrm{Beta}(\mathrm{mode}=0.12, \mathrm{conc}=60)$ with uniform accuracy $a(m)=a_0$; Aggressive uses $m \sim \mathrm{Beta}(0.20, 25)$ with uniform $a(m)=a_0$; Dispersed uses $m \sim \mathrm{Beta}(0.10, 6)$ with linearly declining accuracy $a(m)=\mathrm{clip}(a_0 - 0.20\,m, 0.01, 0.99)$, where $m = |p - q|$ denotes the margin. 

In all cases, $a_0$ is calibrated via Brent's method on mean Brier so that the quadratic score matches the market within $\pm 0.05$ in Regime A/B. The calibrated values are $a_0 = 0.999 / 0.073$ for Conservative, $0.877 / 0.338$ for Aggressive, and $0.971 / 0.473$ for Dispersed, where the first number corresponds to Regime A (model beats market) and the second to Regime B (model loses to market). Empirically, these yield $\Pr[m \le 0.15] = 0.85,\ 0.58,\ 0.68$ and OLS win-rate slopes of $+0.00,\ -0.01,\ -0.20$, respectively.

For Brittle, matching the same $\Delta S$ (Brier) cannot be achieved with a single linear win-rate schedule: slopes steep enough to capture the empirical tail decline reduce mean accuracy below the Regime A target, while flatter slopes eliminate the pattern. We therefore use a two-piece schedule
\[
a(m) \;=\; \mathrm{clip}\bigl(a_0 - s \cdot \max(0,\, m - \tau),\; 0.01,\; 0.99\bigr),
\]
with $\tau = 0.20$, $s = 1.0$, and $m \sim \mathrm{Beta}(0.12, 20)$, calibrated to $a_0 = 0.999 / 0.223$ for Regimes A/B. In \cref{fig:personaArchetypes}, we visualize this fitting a linear regression (OLS) for comparability with the other personas. 

\subsection{Proper betting rule capital allocation surfaces}
\label{app:weight-surfaces}

The three proper scoring rules evaluated -- Brier, logarithmic, and spherical -- differ in \emph{how much capital} they allocate to each bet (i.e., the number of shares purchased). \cref{fig:weight-surfaces} visualizes these weight functions across the full $(q, m)$ plane, where m is the signed margin.

Under the Brier rule, the weight is simply $w = |p - q|$, so capital scales linearly with divergence and is independent of $q$, producing vertical contour lines. As a result, Brier is the most aggressive strategy at large divergences, concentrating capital where model overconfidence is most pronounced. In contrast, the logarithmic rule, $w = |\log p - \log q|$, grows sublinearly in divergence and depends on the price level. The inward curvature of its contours at large margins reflects this diminishing sensitivity, so weight increases more slowly as divergence grows. These features make the log rule more conservative. The spherical rule lies between these extremes: its weight is nonlinear in both divergence and price, allocating more than log at moderate divergences—amplifying smaller, more reliable signals—but less than Brier at extreme divergences. It also exhibits mild price dependence, with slightly higher weights near $q = 0.5$, where normalization effects are weakest.

Generally, Brier favors personas whose accuracy holds across margin sizes (e.g., Conservative), but in the regime in which all personas can outperform the market in terms of accuracy, the performance gap is sufficiently large that even less stable personas benefit from linear scaling. For smaller (though still positive) Brier gaps, sharp deterioration in win rate at higher margins can cause linear scaling to over-weight large-margin forecasts, leading to outsized losses. In such cases, Log may be preferable due to its stronger penalization of extreme probabilities.


\begin{figure}[h]
  \centering
  \includegraphics[width=0.7\linewidth]{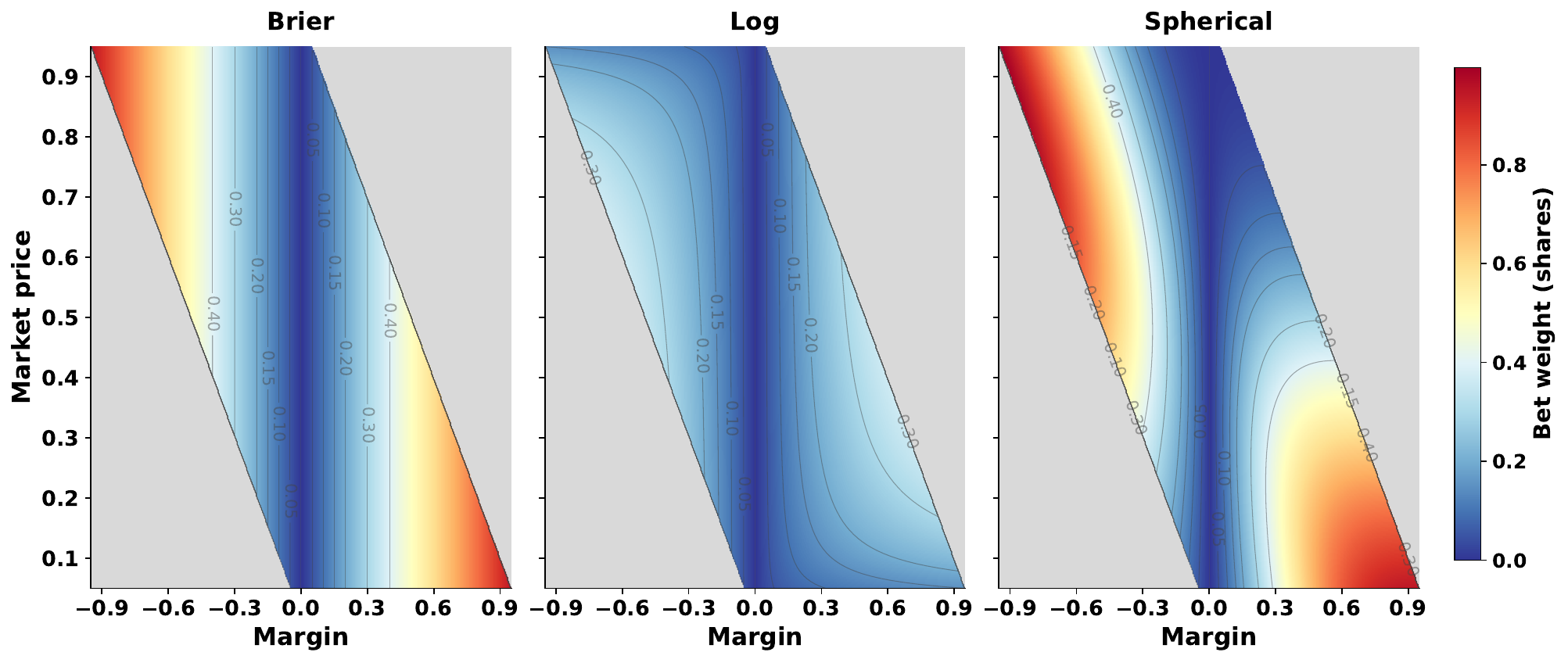}
  \caption{Capital allocation (bet weight in shares) as a function of
  signed margin and market price ($q$, $y$-axis)
  for the three proper betting strategies.  Warmer colors indicate larger
  bets.  Grey regions are infeasible ($p \notin (0,1)$).  Contour lines
  mark constant weight levels.}
  \label{fig:weight-surfaces}
\end{figure}

\subsection{Mapping empirical forecasters to personas}
\label{app:mapPersonas}

As shown in \cref{tab:appEdgeDistribution}, along the margin dimension, models separate into three clear groups: seven place $75\%$--$85\%$ of bets below $0.15$ (Brittle), three distribute more mass across margins (Dispersed, $<70\%$ below $0.15$), and two concentrate over $85\%$ at small margins (Conservative).
Second, \cref{fig:margin-winrate} confirms the persona assignments along the accuracy consistency dimension. The models exhibiting the Brittle persona exhibit a steep win-rate slope ($\beta = -6.8$), with win-rate falling from over $50\%$ to below $30\%$ by margin $0.50$. Dispersed models decline more gradually ($\beta = -3.6$), while Conservative models are essentially flat ($\beta = 0$). 

\begin{table}[h]
  \centering
  \small
  
  \caption{Language model margin concentration. Entries report the share of bets (\%) in each margin bin. $\Delta S$ denotes the aggregate score difference relative to the market. Greyed columns correspond to high-margin bins with sparse data and are excluded from the win-rate analysis. Shares are rounded to the nearest integer; $^{\ast}$ indicates a value that is positive but rounds to $+0.000$.}
  \label{tab:appEdgeDistribution}
  
  \setlength{\tabcolsep}{2.5pt}
  \renewcommand{\arraystretch}{1.08}

\begin{tabular}{@{}lcccccccc@{}}
\toprule
\textbf{Model}
& \textbf{[0,.05)}
& \textbf{[.05,.15)}
& \textbf{[.15,.25)}
& \textbf{[.25,.50)}
& {\color{gray}\textbf{[.50,.70)}}
& {\color{gray}\textbf{[.70,.90)}}
& {\color{gray}\textbf{[.90,1.0)}}
& $\boldsymbol{\Delta S}$ \\
\midrule

GPT-5.2 (High)
& 61 & 23 & 7 & 5
& {\color{gray}2} & {\color{gray}1} & {\color{gray}1}
& $-0.038$ \\

GPT-5.2 (Base)
& 59 & 23 & 7 & 6
& {\color{gray}2} & {\color{gray}2} & {\color{gray}1}
& $-0.044$ \\

Grok 4.1 Fast
& 58 & 22 & 7 & 7
& {\color{gray}3} & {\color{gray}2} & {\color{gray}1}
& $-0.053$ \\

Claude Sonnet 4.5
& 54 & 25 & 8 & 7
& {\color{gray}3} & {\color{gray}2} & {\color{gray}1}
& $-0.055$ \\

Kimi K2 Thinking
& 57 & 24 & 6 & 7
& {\color{gray}2} & {\color{gray}2} & {\color{gray}1}
& $-0.053$ \\

DeepSeek V3.2
& 50 & 25 & 8 & 9
& {\color{gray}4} & {\color{gray}3} & {\color{gray}1}
& $-0.069$ \\

Minimax M2
& 52 & 25 & 8 & 8
& {\color{gray}3} & {\color{gray}3} & {\color{gray}2}
& $-0.064$ \\

\rowcolor{orange!7}
\textbf{\textit{Brittle (avg)}}
& \textbf{56} & \textbf{24} & \textbf{7} & \textbf{7}
& {\color{gray}\textbf{3}} & {\color{gray}\textbf{2}} & {\color{gray}\textbf{1}}
& $\mathbf{-0.054}$ \\

\midrule

LLaMA 4 Maverick
& 46 & 21 & 9 & 10
& {\color{gray}5} & {\color{gray}4} & {\color{gray}5}
& $-0.115$ \\

Qwen 3 235B
& 45 & 21 & 9 & 12
& {\color{gray}5} & {\color{gray}4} & {\color{gray}3}
& $-0.106$ \\

DeepSeek R1
& 37 & 22 & 10 & 14
& {\color{gray}8} & {\color{gray}6} & {\color{gray}5}
& $-0.140$ \\

\rowcolor{teal!7}
\textbf{\textit{Dispersed (avg)}}
& \textbf{43} & \textbf{21} & \textbf{9} & \textbf{12}
& {\color{gray}\textbf{6}} & {\color{gray}\textbf{5}} & {\color{gray}\textbf{4}}
& $\mathbf{-0.120}$ \\

\midrule

Gemini 3
& 81 & 13 & 2 & 1
& {\color{gray}0} & {\color{gray}0} & {\color{gray}2}
& $+0.002$ \\

Claude Opus 4.6
& 69 & 20 & 5 & 3
& {\color{gray}1} & {\color{gray}1} & {\color{gray}0}
& $+0.000^{\ast}$ \\

\rowcolor{purple!7}
\textbf{\textit{Conservative (avg)}}
& \textbf{75} & \textbf{16} & \textbf{4} & \textbf{2}
& {\color{gray}\textbf{0}} & {\color{gray}\textbf{0}} & {\color{gray}\textbf{1}}
& $\mathbf{+0.001}$ \\

\bottomrule
\end{tabular}

\end{table}

\subsection{Details about live pilot test}
\label{app:liveTrading}

We deploy a portfolio-management version of the forecasting framework in \citet{yang2025llm}. The agent runs on a fixed two-hour cadence; on each cycle it (i)~refreshes prices and lifecycle state for every market it currently tracks, (ii)~discovers a batch of new markets, (iii)~queries the LLM forecaster on each eligible market through a standardized prediction context, and (iv)~hands the resulting probabilities to an executor that places limit orders on Kalshi. 

A Kalshi market is pulled for analysis on a given cycle if it is open and its \texttt{close time} lies between 2 and 14 days (we impose a 14-day maximum horizon to facilitate capital turnover and avoid positions being tied up in long-dated markets). To limit exposure to trading fees and LLM-driven variance, a market is skipped on a given cycle if the market price has not moved by at least $10$\textcent\ since the last fill; markets we have never traded on are always re-evaluated.

For each eligible market we issue a single forecast call. The model is given the market's title, resolution rules, and contemporaneous market data and is asked to return a JSON object containing a rationale and the probability that the contract resolves YES. We use the same prompting template and prediction context as \citet{yang2025llm}. The deployed forecaster is Gemini~3 Pro with high-effort thinking and Google Search grounding enabled.

\subsubsection{Prediction context and methodology}
Let $p$ denote the model's probability of YES. We do not trade when $p$ lies within the bid--ask band (a consequence of the occasional non-zero bid-ask spread in real-world prediction markets), as this implies no actionable margin. 

Whenever a market is traded, we rebalance the position toward the new target $\tau$ on each cycle, as is detailed in \cref{sec:portfolio}. 

All orders are placed as limit orders at the prevailing ask, so execution is not guaranteed: on thin Kalshi books, a portion of the requested size often remains unfilled at the end of the cycle. If they do not fill, they expire and are cancelled at the next cycle when updated prices and forecasts produce a new target. The next-cycle delta is then computed based on the positions that actually filled.

\cref{tab:live-trading} presents the full detailed performance and trading statistics for the strategy over the evaluation period. We also provide access to the trading history \href{https://prophets-profit.onrender.com/}{\textcolor{green!50!black}{\textbf{here}}}, which logs performance of the trading agent over the trading period.

\subsubsection{Eligibility criteria}
\label{sec:exclusionCriteria}
Eligibility is determined ex-ante; we exclude three categories of markets and halt trading three hours prior to event resolution, the latter reflecting the reduced informational advantage of LLMs close to resolution time \citep{yang2025llm}. First, we exclude the \textsc{Mentions} category, where contracts resolve based on unstructured public statements (e.g., ``will X mention Y''), as we find that LLMs perform substantially worse on these questions in tests over publicly available data released by \citet{yang2025llm}. Second, we exclude events for which a Kalshi market's listed \texttt{close\_time} is more than one hour after the true resolution time of the event\footnote{Kalshi's \texttt{close\_time} is the trading deadline and does not always coincide with when the underlying event resolves. While these align for most events, some have a \texttt{close\_time} set hours or days after the outcome is publicly known. This discrepancy is observable ex ante via a field in the API deemed the \texttt{expected\_expiration\_time}, which indicates the event's resolution time.}, creating a mismatch between the LLM's belief about market closure and the actual outcome realization. 

In addition to eligibility criteria, we also drop a market from consideration due to an underspecified resolution rule. We have contacted Kalshi about contracts of this kind, and we view such situations as an artifact of the novelty of prediction markets that will phase out as venues mature and resolution-rule conventions standardize; this is not a limitation of our betting strategy but of the underlying contract specification. The event was titled \emph{``What will be the top AI model this month?''} and the published resolution rule reads in full: \emph{``If \texttt{claude-opus-4-6-thinking} is the top-ranked AI model on Apr 30, 2026 at 10:00 AM ET, then the market resolves to Yes.''} The rule does not specify which leaderboard, benchmark, or evaluation suite determines ``top-ranked,'' nor does it name the operator that produces such a ranking. Furthermore, plausible benchmarks --- LMArena, Artificial Analysis, the OpenLLM Leaderboard, MMLU, etc. --- can even disagree on a given day, so the market's outcome is determined by the market operator's discretionary choice of ranking source rather than by any externally verifiable fact. We therefore treat this market as failing the resolution-clarity precondition for forecasting and exclude it from the agent's eligible set \footnote{Subsequent monthly versions of this contract now explicitly specify the LM Arena leaderboard as the resolution source.}.

\subsubsection{Case studies} 
We present two representative case studies drawn from live trading runs to illustrate how the proposed strategy operates in practice. These examples highlight both the decision-making process and the resulting performance.

\paragraph{Case 1: Weekly Executive Order count (\texttt{KXEOWEEK-26APR25-0}).}
This contract asked whether the President of the United States would sign more than 0 Executive Orders between April~19--25, 2026. By the time the agent first evaluated the market on April~21, public news cycles were dominated by a recent executive order (EO) accelerating medical treatments for serious mental illness, and Kalshi traded YES at $0.77$. Gemini~3 priced YES essentially in line with the market, but on April~22 issued a $\texttt{BUY\_NO}$ signal with the rationale:
\begin{quote}\small
``That specific Executive Order (accelerating medical treatments for serious mental illness) was signed on Saturday, April~18, which falls just outside the April~19--25 window. No new Executive Orders have been signed so far between April~19 and April~22, leaving only three days in the market period. The current prediction market probability of $66\%$ is likely inflated by traders conflating the April~18 signing with this week's news cycle, so the true probability of an additional EO being signed in the remaining window is materially lower.''
\end{quote}

Eight prior forecasts in which
model and market had agreed on $p_{\text{yes}} \approx 0.77$, the model's
$p_{\text{yes}}$ collapsed to $0.31$ in a single cycle while Kalshi's
$\texttt{yes\_ask}$ still sat at $0.66$, opening a $35$-cent edge on the
NO side. The strategy converted that edge into $35$ NO shares at $\$0.37$
per share.

As the market then drifted in the model's direction over the following five days---first to $\texttt{no\_ask} = 0.68$ on April~24, then to $\texttt{no\_ask} = 0.99$ by April~27---the position was gradually reduced as the edge decayed. Rather than holding a fixed exposure, the rebalancer continuously reassessed whether additional upside remained. It sold part of the position as soon as the mispricing meaningfully narrowed, taking $7$ shares at $\$0.67$ ($+\$1.87$), and later unwound most of the remaining exposure near certainty at $\$0.98$ ($+\$4.35$), where further gains were negligible.

The remaining $23$ shares, accumulated at an average cost of $\$0.43$, were
held into resolution: the President signed zero Executive Orders that week,
the contract resolved NO, and the residual position paid out at $\$1.00$ per
share for an additional $+\$13.01$. The per-market net P\&L of $+\$19.23$ on
a peak cost basis of $\$15.21$ ($+126\%$) reflects gains accrued both through
early liquidation at improving prices and through the final resolution payoff.

\paragraph{Case 2: U.S.\ Strategic Petroleum Reserve level (\texttt{KXSPRLVL-26APR01-T415}).}
This contract asked whether the U.S.\ Strategic Petroleum Reserve (SPR) level on April~1, 2026, would be above $415$ million barrels. The reserve had been stable at $415.44$ million barrels for several weeks, but on March~11 the Department of Energy announced a $172$-million-barrel emergency drawdown, which led the market to price in a high probability of the level falling below $415$.

By March~28, Gemini~3 was already assigning low probability to YES (around $0.13$), and the agent initially accumulated NO positions. However, on March~30 at 14:06 UTC, the agent reversed sharply after identifying a timing mismatch in how the data is measured. It recognized that the EIA Weekly Petroleum Status Report records inventory as of 7:00 a.m.\ on Fridays, while the early phase of the announced drawdown had not yet been formally recorded in contract awards or physical accounting at that cutoff:

\begin{quote}\small
``The Department of Energy did not award the initial contracts for the first 45.2 million barrels until Friday, March~27. The EIA's Weekly Petroleum Status Report measures inventory strictly as of 7:00~a.m.\ on Fridays, meaning any physical shipments that commenced later that day will not be captured in the data for the week ending March~27. Consequently, the SPR level---which has held perfectly steady at 415.44 million barrels for several weeks---will almost certainly print above 415 in the upcoming release.''
\end{quote}

In other words, although drawdown activity had been announced, it had not yet entered the official measurement window used for settlement. This meant that the upcoming report would still reflect the pre-drawdown level.

Acting on this insight, the agent closed its NO position and bought $81$ YES shares at $\$0.12$. The following day, the EIA release confirmed the SPR level remained at $415.44$ MMbbl, and the contract repriced to $\$0.87$, allowing the agent to sell the full position for a $+\$60.75$ profit on that leg alone (a $+625\%$ return).

After accounting for earlier position adjustments and small late-cycle bets, the net realized P\&L was $+\$45.65$ on a $\$10.92$ cost basis. The key driver of performance was not disagreement about the drawdown itself, but the agent’s correct identification of the measurement cutoff: it exploited the fact that announced changes had not yet entered the reporting window used for settlement.

\subsection{Profitability at extreme margins}
\label{app:profitabilityMargins}

While most large margin bets are wrong (\cref{tab:high-edge-roi} illustrates a win rate $\ll 50\%$ for each high-margin bin for each model), the ROI can remain positive because the payoff on the rare correct bets scales sharply with margin. In these extreme bins, favorable prices mean that a small number of correct, high-conviction bets can more than offset the many losses, yielding positive returns despite low directional accuracy. This effect is most pronounced in the $[0.70,0.9)$ bin for Conservative personas, where even a handful of correct predictions drives substantial gains. We highlight two such cases below from Claude Opus 4.6 and Gemini~3.

\begin{table}[h]
  \centering
  \caption{Per-bin profitability at high margins ($|p-q| \geq 0.50$). For each persona and each margin bin, we report the ($\%$) share of total bets, directional win rate, and ROI under the corresponding scoring rule (Log for Brittle / Dispersed, Brier for Conservative).}
  \label{tab:high-edge-roi}

  \small
  \setlength{\tabcolsep}{6pt}
  \renewcommand{\arraystretch}{1.05}

  \begin{tabular}{llrrr}
    \toprule
    \textbf{Persona} & \textbf{Bin} & \textbf{Share (\%)} & \textbf{Win \%} & \textbf{ROI (\%)} \\
    \midrule

    \multirow{3}{*}{\textit{Brittle}}
    & [.50, .70) & 3 & 16.2 & $-17.3$ \\
    & [.70, .90) & 2 &  7.0 & $-57.4$ \\
    & [.90, 1.0) & 1 & 15.6 & $+191.6$ \\

    \midrule

    \multirow{3}{*}{\textit{Dispersed}}
    & [.50, .70) & 6 & 34.3 & $+0.3$ \\
    & [.70, .90) & 5 & 14.3 & $-25.5$ \\
    & [.90, 1.0) & 4 &  2.7 & $-57.9$ \\

    \midrule

    \multirow{3}{*}{\textit{Conservative}}
    & [.50, .70) & 0 & 17.6 & $+25.8$ \\
    & [.70, .90) & 0 & 10.5 & $+197.8$ \\
    & [.90, 1.0) & 1 &  2.9 & $+76.3$ \\

    \bottomrule
  \end{tabular}
\end{table}

\begin{table}[h]
  \centering
  \caption{Per-market Brier-strategy economics for two extreme-margin case studies. $p_{\text{yes}}$ is the model's predicted probability for YES; $q_{\text{yes}}$ is the YES-side ask price; $y$ is the realized outcome. Weight is the margin $|p - q_{\text{yes}}|$.}
  \label{tab:extreme-margin-cases}

  \small
  \setlength{\tabcolsep}{5pt}
  \renewcommand{\arraystretch}{1.1}

  \begin{tabular}{lccccccc}
    \toprule
    \textbf{Market (Kalshi ID)} 
    & \textbf{Predicted} 
    & \textbf{Resolved} 
    & $p_{\text{yes}}$ 
    & $q_{\text{yes}}$ 
    & $y$ 
    & \textbf{Weight} 
    & \textbf{ROI (\%)} \\
    \midrule

    \multicolumn{8}{c}{\textit{Case 1: Kirk Cousins's next team -- ``Stays with Atlanta or Retires'' (Claude Opus 4.6)}} \\
    \cmidrule(lr){1-8}
    Stays with Atlanta or retires
    & 2026-03-11 & 2026-03-25 
    & $0.03$ & $0.78$ & NO 
    & $0.75$ & $+21\%$ \\

    \midrule

    \multicolumn{8}{c}{\textit{Case 2: Rain in Dallas, January 2026 -- ``Above 2 inches'' (Gemini 3)}} \\
    \cmidrule(lr){1-8}
    Above 2 inches
    & 2026-01-25 & 2026-02-01 
    & $0.22$ & $0.98$ & NO 
    & $0.76$ & $+1900\%$ \\

    \bottomrule
  \end{tabular}
\end{table}

\paragraph{Case 1: Kirk Cousins's next team (\texttt{KXNEXTTEAMNFL-26KCOUSINS}).}
This 32-way market asked which NFL team Kirk Cousins would sign with after his March 11, 2026 release by the Atlanta Falcons. On the ``Stays with Atlanta or Retires'' contract, Kalshi traded YES at $0.78$, while Claude Opus 4.6 assigned $p = 0.03$, an implied disagreement of $0.75$. Its rationale (recorded same-day, 2026-03-11):
\begin{quote}\small
``Kirk Cousins is being released March 11, 2026 and wants to play. The Arizona Cardinals have the clearest QB vacancy after releasing Kyler Murray, with strong scheme fit under new HC Mike LaFleur (Shanahan/McVay system). Multiple insiders from NBC Sports, SI, and PFT have identified Arizona as the top landing spot, with Cousins offering a bridge starter option alongside weapons like Marvin Harrison Jr. Pittsburgh is the main alternative if Rodgers retires, while most other QB-needy teams (Vikings targeting Murray, Dolphins signed Willis, Jets got Geno Smith) have already addressed the position. Cousins may take his time signing, but Arizona's need is clear and urgent.''
\end{quote}

The model's edge here is real-time news synthesis: same-day reports confirming Cousins's intent to keep playing, combined with insider coverage identifying Arizona as the destination, and elimination of competing QB-needy teams that the market price had implicitly assumed were still in play. Cousins did not remain in Atlanta and did not retire. Under the Brier strategy, the position bought $0.75$ NO shares at $\$0.83$, costing $\$0.622$; the contract resolved NO, paying $\$0.75$, for a return of $+21\%$.

\paragraph{Case 2: Rain in Dallas, January 2026 (\texttt{KXRAINDALM-26JAN}).}
This 7-way ladder market asked the cumulative January precipitation in Dallas. On the ``Above 2 inches'' contract, Kalshi priced YES at $0.98$ (with NO at $0.05$), while Gemini~3 assigned $p = 0.22$. Its rationale (recorded 2026-01-25):
\begin{quote}\small
``NWS data confirms recorded precipitation of 0.48 inches through Jan 23, with a winter storm forecast to add $\sim$0.5--0.8 inches of liquid equivalent, making the total likely to reach 1.0--1.5 inches. The provided market data was discounted due to internal inconsistencies (P($>$3) $>$ P($>$2)) and contradiction with verified meteorological reports. Probabilities favor exceeding 1 inch but decline sharply for 2+ inches given the dry start to the month and limited remaining forecast precipitation.''
\end{quote}

The actual January total fell short of two inches. The Brier strategy bought $0.76$ NO shares at $0.05$, staking $0.038$ and paying out $0.76$—a return of $+1900\%$. Two factors made the disagreement actionable: (i) verifiable third-party meteorology directly contradicted the market's near-certainty, and (ii) the model identified an internal inconsistency in the order book—specifically, the probability for $P(>3)$ inches was greater than that of $P(>2)$ inches, violating monotonicity (due to factors like thin liquidity or delayed updates across related contracts). By detecting this structural mismatch, the model effectively inferred that at least one of the prices must be miscalibrated, and correspondingly discounted the market signal rather than treating it as fully informative.

In both cases, the model’s advantage stems from combining verifiable public data with effective information synthesis. The models identify and integrate disparate signals -- real-time reporting (Cousins), structured third-party data (NWS) -- into coherent forecasts that the market had not yet fully incorporated. The Dallas precipitation case further illustrates an additional capability: detecting and reasoning about internal inconsistencies in market prices themselves.

As such, these examples suggest that LLM forecasters are particularly well-suited to settings with a large, readily available body of public information that is dispersed across sources, or where market prices contain internal inconsistencies. In these environments, LLMs can aggregate, cross-reference, and reconcile information more efficiently than humans, forming a coherent view faster than the market updates. 

\end{document}